\documentclass{article}
\usepackage{amsmath}
\usepackage{amsbsy,amssymb,amsthm,commath,dsfont}
\usepackage{algorithm,caption}
\usepackage{algorithmic}
\usepackage[usenames,dvipsnames]{pstricks}
\usepackage{pst-grad} 
\usepackage{pst-plot} 
\usepackage{graphicx}
\usepackage{natbib}
\usepackage{xspace}
\usepackage{enumitem}
\usepackage{fullpage}
\usepackage{color}

\usepackage[final]{nips_2016}
\usepackage{eqparbox}

\usepackage[utf8]{inputenc} 
\usepackage[T1]{fontenc}    
\usepackage{hyperref}       
\usepackage{url}            
\usepackage{booktabs}       
\usepackage{amsfonts}       
\usepackage{nicefrac}       
\usepackage{microtype}      

\renewcommand{\algorithmiccomment}[1]{\bgroup\hfill\#~#1\egroup}

\newcommand\SEARCH{\ensuremath{\operatorname{\textsc{Search}}}\xspace}
\newcommand\CCQ{\ensuremath{\operatorname{\textsc{CCQ}}}\xspace}
\newcommand\ECCQ{\ensuremath{\operatorname{\textsc{ECCQ}}}\xspace}
\newcommand\LABEL{\ensuremath{\operatorname{\textsc{Label}}}\xspace}
\newcommand\SAL{\ensuremath{\operatorname{\textsc{Sample-and-Label}}}\xspace}
\newcommand\UVS{\ensuremath{\operatorname{\textsc{Upgrade-Version-Space}}}\xspace}
\newcommand\PVS{\ensuremath{\operatorname{\textsc{Prune-Version-Space}}}\xspace}
\newcommand\EC{\ensuremath{\operatorname{\textsc{Error-Check}}}\xspace}

\newcommand\AL{\ensuremath{\operatorname{AL}}\xspace}
\newcommand\CAL{\ensuremath{\operatorname{CAL}}\xspace}

\newcommand\DIS{\ensuremath{\operatorname{Dis}}}
\newcommand\B{\ensuremath{\operatorname{B}}}
\newcommand\err{\ensuremath{\operatorname{err}}}
\newcommand\E{\ensuremath{\operatorname{E}}}

\def\calX{\mathcal{X}}
\def\calY{\mathcal{Y}}

\newcommand\ind[1]{\ensuremath{\mathds{1}\{#1\}}}
\newcommand{\hide}[1]{}

\newcommand\Hlabel{\ensuremath{H_{\LABEL}}}
\newcommand\Hsearch{\ensuremath{H_{\SEARCH}}}

\newcommand\algr{\ensuremath{\operatorname{\textsc{Larch}}}\xspace}
\newcommand\alga{\ensuremath{\operatorname{\textsc{A-Larch}}}\xspace}
\newcommand\algaa{\ensuremath{\operatorname{\textsc{AA-Larch}}}\xspace}
\newcommand\algrb{\ensuremath{\operatorname{\textsc{Seabel}}}\xspace}

\DeclareMathOperator*{\argmin}{arg\,min}
\DeclareMathOperator{\polylog}{polylog}
\DeclareMathOperator{\poly}{poly}

\newtheorem{lemma}{Lemma}
\newtheorem{theorem}{Theorem}
\newtheorem{fact}{Fact}
\newtheorem{claim}{Claim}
\newtheorem{proposition}{Proposition}

\title{Search Improves Label for Active Learning}

\author{
Alina Beygelzimer \\
Yahoo Research \\
New York, NY \\
\texttt{beygel@yahoo-inc.com} \\
\And
Daniel Hsu \\
Columbia University \\
New York, NY \\
\texttt{djhsu@cs.columbia.edu} \\
\AND
John Langford \\
Microsoft Research \\
New York, NY \\
\texttt{jcl@microsoft.com} \\
\And
Chicheng Zhang \\
UC San Diego \\
La Jolla, CA \\
\texttt{chz038@cs.ucsd.edu}
}

\newcommand{\cz}[1]{ \textcolor{magenta} {CZ: #1}}

\begin{document}
\maketitle
\begin{abstract} We investigate active learning with access to two distinct
  oracles: \LABEL (which is standard) and \SEARCH (which is not).  The
  \SEARCH oracle models the situation where a human searches a
  database to seed or counterexample an existing solution.  \SEARCH is
  stronger than \LABEL while being natural to implement in many
  situations.  We show that an algorithm using both oracles can provide
  exponentially large problem-dependent improvements over \LABEL alone.
\end{abstract}

\section{Introduction}
Most active learning theory is based on interacting with a 
\LABEL oracle:  An active learner observes unlabeled examples, 
each with a label that is initially hidden.  The learner provides an unlabeled
example to the oracle, and the oracle responds with the label.
Using \LABEL in an active learning algorithm is known to give 
(sometimes exponentially large)
problem-dependent improvements in label complexity, even in the agnostic
setting where no assumption is made about the underlying
distribution~\citep[e.g.,][]{balcan2006agnostic,Hanneke07,DHM07,hanneke15}.

A well-known deficiency of \LABEL arises in the presence of
rare classes in classification problems, frequently the case in
practice~\citep{AP10,ICE}.
Class imbalance may be so extreme that simply {finding} an example
from the rare class can exhaust the labeling budget.
Consider the problem of learning interval functions
in $[0,1]$.  Any \LABEL-only active learner needs at least $\Omega(1/\epsilon)$ \LABEL queries to learn an arbitrary target interval with error at most
$\epsilon$~\citep{dasgupta2005coarse}.
Given any positive example from the interval, however,
the query complexity
of learning intervals collapses to $O(\log(1/\epsilon))$, as 
we can just do a binary search for each of the end points.

A natural approach used to overcome this hurdle in practice is to
search for known examples of the rare class \citep{AP10,ICE}.
Domain experts are often
adept at finding examples of a class by various, often clever means.
For instance, when building a hate speech filter, a simple web search can readily produce a set of positive examples. Sending a random batch of 
unlabeled text to \LABEL is unlikely to
produce any positive examples at all. 


Another form of interaction common in practice is providing counterexamples
to a learned predictor.  When monitoring the
stream filtered by the current hate speech filter, a human editor
may spot a clear-cut example of hate speech that seeped through the filter.
The editor, using all the search tools available to her,
may even be tasked with searching for such counterexamples.
The goal of the learning system is then to interactively
restrict the searchable space, guiding the search process to where 
it is most effective. 

Counterexamples can be ineffective or misleading in practice as well.
Reconsidering the intervals example above, a counterexample on the
boundary of an incorrect interval provides no useful information about
any other examples.  What is a good counterexample?  What is a natural
way to restrict the searchable space?  How can the intervals problem
be generalized?

We define a new oracle,
\SEARCH, that provides {counterexamples} to {version
  spaces}.  Given a set of possible classifiers $H$ mapping unlabeled
examples to labels, a \emph{version space} $V \subseteq H$ is the subset
of classifiers still under consideration by the algorithm.  
A \emph{counterexample} to
a version space is a labeled example which every classifier in the
version space classifies incorrectly. When there is no
counterexample to the version space, \SEARCH returns nothing.

How can a counterexample to the version space be used?
We consider a nested sequence of hypothesis classes of increasing
complexity, akin to Structural Risk Minimization (SRM) in passive learning
\citep[see, e.g.,][]{Vapnik82,Devroye96}.
When \SEARCH produces a counterexample to the version space, it
gives a proof that the current hypothesis class is too
simplistic to solve the problem effectively.  
We show
that this guided increase in hypothesis complexity results in a radically
lower \LABEL complexity than directly learning on the complex space.
Sample complexity bounds for model selection in \LABEL-only active learning
were studied by \citet{BHV10,Hanneke11}. 

\SEARCH can easily model the practice of seeding discussed earlier.
If the first hypothesis class has just the constant always-negative
classifier
$h(x) = -1$, a seed example with label $+1$ is a counterexample to the
version space.
Our most basic algorithm uses \SEARCH just once before using \LABEL, but it is
clear from inspection that multiple seeds are not harmful, and they may be
helpful if they provide the proof required to operate with an appropriately
complex hypothesis class.

Defining \SEARCH with respect to a version space rather than a single
classifier allows us to formalize ``counterexample far from the
boundary'' in a general fashion which is compatible with the way
\LABEL-based active learning algorithms work.  

\paragraph{Related work.}
The closest oracle considered in the literature is the Class
Conditional Query (\CCQ)~\citep{BH12} oracle.  A query to \CCQ
specifies a finite set of unlabeled examples and a label while
returning an example in the subset with the specified label, if one
exists.

In contrast, \SEARCH has an implicit query set that is an entire region of the
input space rather than a finite set.  Simple searches over this large implicit domain can more plausibly discover relevant counterexamples:
When building a detector for penguins in images,
the input to \CCQ might be a set of images and the label ``penguin''.
Even if we are very lucky and the set happens to contain a penguin image,
 a search amongst image tags may fail to find it
in the subset because it is not tagged appropriately. 
\SEARCH is more likely to discover
counterexamples---surely there are many images correctly tagged 
as having penguins.

Why is it natural to define a query region implicitly via a version space?
There is a practical reason---it is a concise description of a natural region
with an efficiently implementable membership filter~\citep{BHLZ,BHKLZ,HAHLS}.
(Compare this to an oracle call that has to explicitly enumerate a large set of
examples.  The algorithm of~\citet{BH12} uses samples of size roughly $d\nu
/\epsilon^2$.) 



The use of \SEARCH in this paper is also substantially
different from the use of \CCQ by~\citet{BH12}.
Our motivation is to use \SEARCH to assist \LABEL, as opposed to using \SEARCH
alone.
This is especially useful in any setting where the cost of \SEARCH is
significantly higher than the cost of \LABEL---we hope to avoid using \SEARCH
queries whenever it is possible to make progress using \LABEL queries.
This is consistent with how interactive learning systems are used in practice.
For example, the Interactive Classification and Extraction system of~\citet{ICE}
combines \LABEL with search in a production environment.

The final important distinction is that we require \SEARCH to return the label
of the optimal predictor in the nested sequence.  For many natural sequences of
hypothesis classes, the Bayes optimal classifier is eventually in the sequence,
in which case it is equivalent to assuming that the label in a counterexample is
the most probable one, as opposed to a randomly-drawn label from the conditional
distribution (as in \CCQ and \LABEL).

Is this a reasonable assumption?
Unlike with \LABEL queries, where the labeler has no choice of what
to label, here the labeler \emph{chooses} a counterexample.
If a human editor finds an unquestionable example of hate speech that  
seeped through the filter, it is quite reasonable to assume that 
this counterexample is
consistent with the Bayes optimal predictor for any sensible 
feature representation.


\paragraph{Organization.}
Section~\ref{sec:defs} formally introduces the setting.
Section~\ref{sec:oracles} shows that \SEARCH is at least as powerful as \LABEL.
Section~\ref{sec:realizable} shows how to use \SEARCH and \LABEL
jointly in the realizable setting where a zero-error classifier
exists in the nested sequence of hypothesis classes.
Section~\ref{sec:agnostic} handles the agnostic setting where \LABEL
is subject to label noise, and shows an amortized approach to combining the two
oracles with a good guarantee on the total cost.
%

\section{Definitions and Setting}
\label{sec:defs}
In active learning, there is an underlying distribution $D$ over
$\calX\times \calY$, where $\calX$ is the instance space and
$\calY := \{-1,+1\}$ is the label space.
The learner can obtain independent draws from $D$, but the label is
hidden unless explicitly requested through a query to the \LABEL
oracle.
Let $D_\calX$ denote the marginal of $D$ over $\calX$.

We consider learning with a nested sequence of hypotheses classes $H_0 \subset
H_1\subset \dotsb \subset H_k \dotsb$, where $H_k \subseteq \calY^\calX$ has VC
dimension $d_k$.
For a set of labeled examples $S \subseteq \calX \times \calY$, let $H_k(S) :=
\{ h \in H_k : \forall (x,y) \in S \centerdot h(x) = y \}$ be the set of
hypotheses in $H_k$ consistent with $S$.
Let $\err(h) := \Pr_{(x,y)\sim D}[h(x)\not=y]$ denote the error rate of a
hypothesis $h$ with respect to distribution $D$, and $\err(h,S)$ be the error
rate of $h$ on the labeled examples in $S$.
%
Let $h_k^* = \argmin_{h \in H_k} \err(h)$ breaking ties arbitrarily
and let $k^* := \argmin_{k\geq0} \err(h_k^*)$ breaking ties in favor
of the smallest such $k$.  For simplicity, we assume the minimum is
attained at some finite $k^*$.  Finally, define $h^* := h_{k^*}^*$,
the optimal hypothesis in the sequence of classes.
%
The goal of the learner is to learn a hypothesis with error rate not
much more than that of $h^*$.

In addition to \LABEL, the learner can also query $\SEARCH$ with a version space.
\vskip -.0in
\begin{algorithm}[H]
 \floatname{algorithm}{Oracle}
 \caption*{{\bf Oracle} $\SEARCH_H(V)$ (where $H \in \cbr{H_k}_{k=0}^\infty$)}
  \begin{algorithmic}[1]
    \REQUIRE Set of hypotheses $V\subset H$
    \ENSURE \mbox{Labeled example $(x,h^*(x))$
    s.t.~$h(x) \neq h^*(x)$} for all $h\in V$, or $\bot$ if there is
    no such example.
  \end{algorithmic}
\end{algorithm}

\vskip -.15in
Thus if $\SEARCH_H(V)$ returns an example, this example is
a \emph{systematic} mistake made by all hypotheses in $V$.
(If $V=\emptyset$, we expect $\SEARCH$ to return some example, i.e.,
not $\bot$.)

Our analysis is given in terms of the \emph{disagreement coefficient}
of \citet{Hanneke07}, which has been a central parameter for analyzing
active learning algorithms.
Define the \emph{region of disagreement} of a set of hypotheses $V$ as
 $\DIS(V)
  :=
  \{x\in \calX : \exists h,h' \in V \ \text{s.t.}\ h(x)\not=h'(x)\}$.
The \emph{disagreement coefficient of $V$ at scale $r$} is 
 $\theta_V(r)
  := \sup_{h \in V, r' \geq r} \Pr_{D_\calX}[\DIS(\B_V(h,r'))] / r'$,
where $\B_V(h,r')=\{h '\in V: \Pr_{x \sim D_\calX}[h'(x)\not=h(x)] \leq r'
\}$ is the ball of radius $r'$ around $h$.


The $\tilde{O}(\cdot)$ notation hides factors that are polylogarithmic
in $1/\delta$ and quantities that do appear, where $\delta$ is the usual
confidence parameter.

\section{The Relative Power of the Two Oracles}
\label{sec:oracles}

Although \SEARCH cannot always implement \LABEL efficiently, it is as effective at reducing the region of disagreement.
The clearest example is learning threshold classifiers $H := \{h_w: w\in
[0,1]\}$ in the realizable case, where $h_w(x)=+1$ if $w \leq x \leq
1$, and $-1$ if $0\leq x < w$.  A simple binary search with \LABEL
achieves an exponential improvement in query complexity over passive
learning. The agreement region of any set of threshold classifiers
with thresholds in $\intcc{w_{\min}, w_{\max}}$ is $\intco{0,w_{\min}}
\cup \intcc{w_{\max},1}$.  Since \SEARCH is allowed to return any
counterexample in the agreement region, there is no mechanism for
forcing \SEARCH to return the label of a particular point we want.
However, this is not needed to achieve logarithmic query complexity
with \SEARCH: If binary search starts with querying the label of $x\in
[0,1]$, we can query $\SEARCH_H(V_x)$, where $V_x := \{ h_w \in H : w <
x\}$ instead.
If $\SEARCH$ returns $\bot$, we know that the target $w^* \leq x$ and
can safely reduce the region of disagreement to $\intco{0,x}$.  If
$\SEARCH$ returns a counterexample $(x_0,-1)$ with $x_0 \geq x$, we
know that $w^* > x_0$ and can reduce the region of disagreement to
$\intoc{x_0,1}$.

This observation holds more generally.
In the proposition below, we assume
that $\LABEL(x) = h^*(x)$ for simplicity.
If $\LABEL(x)$ is noisy, the proposition holds for
any active learning algorithm that doesn't eliminate any
$h\in H: h(x)=\LABEL(x)$ from the version space.

\begin{proposition}\label{prop:search}
For any call $x\in \calX$ to \LABEL such that $\LABEL(x) = h^*(x)$, 
we can construct a call to \SEARCH
that achieves a no lesser reduction in the region of disagreement.
\end{proposition}
\begin{proof} 
For any $V\subseteq H$, let $\Hsearch(V)$ be the hypotheses in $H$
consistent with the output of $\SEARCH_H(V)$:
if $\SEARCH_H(V)$ returns a counterexample $(x,y)$ to $V$, then
$\Hsearch(V) := \{h \in H : h(x)= y\}$; otherwise, $\Hsearch(V) := V$.
Let $\Hlabel(x) := \{h\in H : h(x) =
\LABEL(x)\}$. Also, let $V_x \ := \ H_{+1}(x) \ := \ \{h\in H : h(x)=+1\}$.
We will show that $V_x$ is such that
%
$\Hsearch(V_x) \ \subseteq \ \Hlabel(x)$,
and hence
$\DIS(\Hsearch(V_x)) \ \subseteq \ \DIS(\Hlabel(x))$.

There are two cases to consider:
If $h^*(x)=+1$, then $\SEARCH_H(V_x)$ returns
$\bot$.  In this case, $\Hlabel(x)=\Hsearch(V_x)=H_{+1}(x)$, and we are done.
If $h^*(x)=-1$, $\SEARCH(V_x)$ returns a valid counterexample
(possibly $(x,-1)$) in the region of agreement of
$H_{+1}(x)$, eliminating all of $H_{+1}(x)$.
Thus $\Hsearch(V_x)\subset H\setminus H_{+1}(x) = \Hlabel(x)$, and the claim
holds also.
%
\end{proof}




As shown by the problem of learning intervals on the line, \textsc{SEARCH} can be exponentially more powerful than \textsc{LABEL}.

\section{Realizable Case}
\label{sec:realizable}

We now turn to general active learning algorithms that combine \SEARCH
and \LABEL.  We focus on algorithms using \emph{both} \SEARCH
and \LABEL since \LABEL is typically easier to implement than \SEARCH
and hence should be used where \SEARCH has no significant advantage.
(Whenever \SEARCH is less expensive than \LABEL, Section~\ref{sec:oracles}
suggests a transformation to a \SEARCH-only algorithm.)

This section considers the realizable case, in which we assume that
the hypothesis $h^* = h_{k^*}^* \in H_{k^*}$ has $\err(h^*) = 0$.
This means that $\LABEL(x)$ returns $h^*(x)$ for any $x$ in the support of
$D_\calX$.
%

\subsection{Combining \textsc{LABEL} and \textsc{SEARCH}}

Our algorithm (shown as Algorithm~\ref{alg:realizable}) is called
\algr, because it combines \LABEL and \SEARCH.
Like many selective sampling methods, \algr uses a version space to determine its
\LABEL queries.

For concreteness, we use (a variant of) the algorithm
of~\citet{CAL94}, denoted by \CAL, as a subroutine in \algr.
The inputs to \CAL are: a version space $V$, the \LABEL oracle, a
target error rate, and a confidence parameter; and its output is a set
of labeled examples (implicitly defining a new version space).
\CAL is described in Appendix~\ref{sec:cal}; its essential
properties are specified in Lemma~\ref{lem:cal}.

\begin{algorithm}[t]
  \caption{\algr}
  \begin{algorithmic}[1]
    \REQUIRE Nested hypothesis classes $H_0 \subset H_1 \subset
    \dotsb$; oracles $\LABEL$ and $\SEARCH$;
    learning parameters $\epsilon,\delta \in (0,1)$

    \STATE \textbf{initialize}
      $S \gets \emptyset$, (index) $k \gets 0$, $\ell \gets 0$
    \FOR{$i = 1, 2, \dotsc$}
      \STATE $e \gets \SEARCH_{H_k}(H_k(S))$
      \label{step:realizable-search}

      \IF[no counterexample found]{$e = \bot$} 
        \IF{$2^{-\ell} \leq \epsilon$}
          \RETURN any $h\in H_k(S)$
          \label{step:realizable-halt}
        \ELSE
          \STATE $\ell \gets \ell + 1$
          \label{step:realizable-halve-error}
        \ENDIF
        \ELSE[counterexample found]
        \STATE $S \gets S \cup \{e\}$
        \STATE $k \gets \min\{ k' : H_{k'}(S) \neq \emptyset \}$
        \label{step:realizable-advance-k}
      \ENDIF
      \STATE $S \gets S \cup \CAL(H_k(S),\LABEL,2^{-\ell},\delta/(i^2+i))$
      \label{step:realizable-cal}

%
%

    \ENDFOR

  \end{algorithmic}
\label{alg:realizable}
\end{algorithm}

\algr differs from \LABEL-only active learners (like \CAL) by first
calling \SEARCH in Step~\ref{step:realizable-search}.
If \SEARCH returns $\bot$, \algr checks to see if the last call to
\CAL resulted in a small-enough error, halting if so in
Step~\ref{step:realizable-halt}, and decreasing the allowed error rate
if not in Step~\ref{step:realizable-halve-error}.
If \SEARCH instead returns a counterexample, the hypothesis class
$H_k$ must be impoverished, so in
Step~\ref{step:realizable-advance-k}, \algr increases the complexity
of the hypothesis class to the minimum complexity sufficient to
correctly classify all known labeled examples in $S$.
After the \SEARCH, \CAL is called in Step~\ref{step:realizable-cal} to
discover a sufficiently low-error (or at least low-disagreement)
version space with high probability.

When \algr advances to index $k$ (for any $k \leq k^*$), its set of
labeled examples $S$ may imply a version space $H_k(S) \subseteq H_k$
that can be actively-learned more efficiently than the whole of $H_k$.
In our analysis, we quantify this through the disagreement coefficient
of $H_k(S)$, which may be markedly smaller than that of the full $H_k$.

The following theorem bounds the oracle query complexity of
Algorithm~\ref{alg:realizable} for learning with both \SEARCH and
\LABEL in the realizable setting.  The proof is in section~\ref{sec:realizable-proof}.

\begin{theorem}
  \label{thm:realizable}
  Assume that $\err(h^*) = 0$.
  For each $k' \geq 0$, let $\theta_{k'}(\cdot)$ be the disagreement
  coefficient of $H_{k'}(S_{[k']})$, where $S_{[k']}$ is the set of
  labeled examples $S$ in \algr at the first
  time that $k \geq k'$.
  Fix any $\epsilon, \delta \in (0,1)$.
  If \algr is run with inputs hypothesis classes
  $\cbr{H_k}_{k=0}^\infty$, oracles \LABEL and \SEARCH, and learning parameters
  $\epsilon,\delta$, then with probability at least $1-\delta$:
  \algr halts after at most $k^* + \log_2(1/\epsilon)$ for-loop
  iterations and returns a classifier with error rate at most $\epsilon$;
  furthermore, it draws at most $\tilde{O}(k^* d_{k^*}/\epsilon)$ unlabeled
  examples from $D_\calX$, makes at most $k^* + \log_2(1/\epsilon)$ queries to
  \SEARCH, and at most
   $\tilde{O}\del[0]{
      \del{ k^* + \log(1/\epsilon) }
      \cdot \del[0]{ \max_{k' \leq k^*} \theta_{k'}(\epsilon) }
      \cdot d_{k^*}
      \cdot \log^2(1/\epsilon)
    }$
  queries to \LABEL.
%
\end{theorem}

\paragraph{Union-of-intervals example.}
We now show an implication of Theorem~\ref{thm:realizable} in the case
where the target hypothesis $h^*$ is the union of non-trivial
intervals in $\calX := [0,1]$, assuming that $D_\calX$ is uniform.
For $k\geq0$, let $H_k$ be the hypothesis class of the union of up to $k$
intervals in $[0,1]$ with $H_0$ containing only the always-negative hypothesis.
(Thus, $h^*$ is the union of $k^*$ non-empty intervals.)
The disagreement coefficient of $H_1$ is $\Omega(1/\epsilon)$, and hence
\LABEL-only active learners like \CAL are not very effective at learning with
such classes.
However, the first \SEARCH query by \algr provides a counterexample to
$H_0$, which must be a positive example $(x_1,+1)$. Hence,
$H_1(S_{[1]})$ (where $S_{[1]}$ is defined in
Theorem~\ref{thm:realizable}) is the class of intervals that
contain $x_1$ with disagreement coefficient $\theta_1 \leq 4$.

Now consider the inductive case.
Just before \algr advances its index to a value $k$ (for any $k
\leq k^*$), \SEARCH returns a counterexample $(x,h^*(x))$ to the
version space; every hypothesis in this version space (which could be
empty) is a union of fewer than $k$ intervals.
If the version space is empty, then $S$ must already contain positive
examples from at least $k$ different intervals in $h^*$ and at least $k-1$ negative examples separating them.
If the version space is not empty, then the point $x$ is either a positive example belonging to a previously uncovered interval in $h^*$ or a negative
example splitting an existing interval. In either case, 
$S_{[k]}$ contains positive examples from at least $k$ distinct
intervals separated by at least $k-1$ negative examples.
The disagreement coefficient of the set of unions of $k$ intervals
consistent with $S_{[k]}$ is at most $4k$,
independent of $\epsilon$.

The VC dimension of $H_k$ is $O(k)$, so
Theorem~\ref{thm:realizable} implies that with high probability,
\algr makes at most $k^* + \log(1/\epsilon)$
queries to \SEARCH and $\tilde{O}((k^*)^3\log(1/\epsilon) +
(k^*)^2 \log^3(1/\epsilon))$ queries to \LABEL.

\subsection{Proof of Theorem~\ref{thm:realizable}}
\label{sec:realizable-proof}

The proof of Theorem~\ref{thm:realizable} uses the following lemma
regarding the \CAL subroutine, proved in Appendix~\ref{sec:cal}.
It is similar to a result of~\citet{Hanneke11}, but an important
difference here is that the input version space $V$ is not assumed to
contain $h^*$.

\begin{lemma}
  \label{lem:cal}
  Assume $\LABEL(x) = h^*(x)$ for every $x$ in the support of
  $D_\calX$.
  For any hypothesis set $V \subseteq \calY^\calX$ with VC dimension
  $d<\infty$, and any $\epsilon,\delta \in (0,1)$, the following holds
  with probability at least $1-\delta$.
  $\CAL(V,\LABEL,\epsilon,\delta)$ returns labeled examples $T
  \subseteq \{ (x,h^*(x)) : x \in \calX \}$ such that for any $h$ in
  $V(T)$, $\Pr_{(x,y)\sim D}[h(x) \neq y\ \wedge \ x \in \DIS(V(T))]
  \leq \epsilon$; furthermore, it draws at most $\tilde{O}(d/\epsilon)$
  unlabeled examples from $D_\calX$, and makes at most $\tilde{O}\del[0]{
  \theta_V(\epsilon) \cdot d \cdot \log^2(1/\epsilon) }$ queries to \LABEL.
\end{lemma}

We now prove Theorem~\ref{thm:realizable}.
By Lemma~\ref{lem:cal} and a union bound, there is an event with
probability at least $1-\sum_{i\geq1} \delta/(i^2+i) \geq 1-\delta$
such that each call to \CAL made by \algr
satisfies the high-probability guarantee from Lemma~\ref{lem:cal}.
We henceforth condition on this event.

We first establish the guarantee on the error rate of a hypothesis
returned by \algr.
By the assumed properties of \LABEL and \SEARCH, and the
properties of \CAL from Lemma~\ref{lem:cal}, the labeled examples
$S$ in \algr are always consistent with
$h^*$.
Moreover, the return property of \CAL implies that at the end of any
loop iteration, with the present values of $S$, $k$, and $\ell$, we
have $\Pr_{(x,y)\sim D}[h(x) \neq y \wedge x \in \DIS(H_k(S))] \leq
2^{-\ell}$ for all $h \in H_k(S)$.
(The same holds trivially before the first loop iteration.)
Therefore, if \algr halts and returns a
hypothesis $h \in H_k(S)$, then there is no counterexample to
$H_k(S)$, and $\Pr_{(x,y)\sim D}[h(x) \neq y \wedge x \in
\DIS(H_k(S))] \leq \epsilon$.
These consequences and the law of total probability imply $\err(h) =
\Pr_{(x,y)\sim D}[ h(x) \neq y \wedge x \in \DIS(H_k(S)) ] \leq
\epsilon$.

We next consider the number of for-loop iterations executed by
\algr.
Let $S_i$, $k_i$, and $\ell_i$ be, respectively, the values of $S$,
$k$, and $\ell$ at the start of the $i$-th for-loop iteration in
\algr.
We claim that if \algr does not halt in the
$i$-th iteration, then one of $k$ and $\ell$ is incremented by at
least one.
Clearly, if there is no counterexample to $H_{k_i}(S_i)$ and
$2^{-\ell_i} > \epsilon$, then $\ell$ is incremented by one
(Step~\ref{step:realizable-halve-error}).
If, instead, there is a counterexample $(x,y)$, then $H_{k_i}(S_i
\cup \{(x,y)\}) = \emptyset$, and hence $k$ is incremented to some
index larger than $k_i$ (Step~\ref{step:realizable-advance-k}).
This proves that $k_{i+1} + \ell_{i+1} \geq k_i + \ell_i + 1$.
We also have $k_i \leq k^*$, since $h^* \in H_{k^*}$
is consistent with $S$, and $\ell_i \leq
\log_2(1/\epsilon)$, as long as \algr does
not halt in for-loop iteration $i$.
So the total number of for-loop iterations is at most $k^* +
\log_2(1/\epsilon)$.
Together with Lemma~\ref{lem:cal}, this bounds the number of
unlabeled examples drawn from $D_\calX$.

Finally, we bound the number of queries to \SEARCH and \LABEL.
The number of queries to \SEARCH is the same as the number of
for-loop iterations---this is at most $k^* + \log_2(1/\epsilon)$.
By Lemma~\ref{lem:cal} and the fact that $V(S' \cup S'') \subseteq
V(S')$ for any hypothesis space $V$ and sets of labeled examples
$S', S''$, the number of \LABEL queries made by \CAL in the $i$-th
for-loop iteration is at most
$\tilde{O}(
  \theta_{k_i}(\epsilon)
  \cdot
  d_{k_i}
  \cdot
  \ell_i^2
  \cdot
  \polylog(i)
)$.
The claimed bound on the number of \LABEL queries made by
\algr now readily follows by taking a
$\max$ over $i$, and using the facts that $i \leq k^*$ and $d_{k'}
\leq d_{k^*}$ for all $k' \leq k$.
\qed

\subsection{An Improved Algorithm}

\algr is somewhat conservative in its use of \SEARCH, interleaving
just one \SEARCH query between sequences of \LABEL queries (from
\CAL).  Often, it is advantageous to advance to higher complexity
hypothesis classes quickly, as long as there is justification to
do so.  Counterexamples from \SEARCH provide such justification,
and a $\bot$ result from \SEARCH also provides useful feedback about
the current version space: outside of its disagreement region, the
version space is in complete agreement with $h^*$ (even if the version
space does not contain $h^*$).  Based on these observations, we
propose an improved algorithm for the realizable setting, which we
call \algrb.  Due to space limitations, we present it in
Appendix~\ref{sec:seabel}.  We prove the following performance
guarantee for \algrb.

\begin{theorem}
  \label{thm:seabel}
  Assume that $\err(h^*) = 0$.
  Let $\theta_k(\cdot)$ denote the disagreement coefficient of $V_i^{k_i}$ at
  the first iteration $i$ in \algrb where $k_i \geq k$.
  Fix any $\epsilon, \delta \in (0,1)$.
  If \algrb is run with inputs hypothesis classes
  $\cbr{H_k}_{k=0}^{\infty}$, oracles \SEARCH and \LABEL, and learning
  parameters $\epsilon,\delta \in (0,1)$, then with probability $1-\delta$:
  \algrb halts and returns a classifier with error rate at most
  $\epsilon$; furthermore, it draws at most $\tilde{O}((d_{k^*} + \log
  k^*)/\epsilon)$ unlabeled examples from $D_\calX$, makes at most $k^* +
  O\del[0]{ \log(d_{k^*}/\epsilon) + \log\log k^* }$ queries to \SEARCH, and at
  most
  $\tilde{O}\del[0]{
    \max_{k \leq k^*} \theta_k(2\epsilon) \cdot
    \del[0]{ d_{k^*} \log^2(1/\epsilon) + \log k^* }
  }$
  queries to \LABEL.
\end{theorem}
It is not generally possible to directly compare Theorems~\ref{thm:realizable}
and~\ref{thm:seabel} on account of the algorithm-dependent disagreement
coefficient bounds.
However, in cases where these disagreement coefficients are comparable (as in
the union-of-intervals example), the \SEARCH
complexity in Theorem~\ref{thm:seabel} is slightly higher (by additive log
terms), but the \LABEL complexity is smaller than that from
Theorem~\ref{thm:realizable} by roughly a factor of $k^*$.
For the union-of-intervals example, \algrb
would learn target union of $k^*$ intervals with $k^* + O(\log(k^*/\epsilon))$
queries to \SEARCH and $\tilde{O}((k^*)^2 \log^2(1/\epsilon))$ queries to
\LABEL.

\section{Non-Realizable Case}
\label{sec:agnostic}

In this section, we consider the case where the optimal hypothesis
$h^*$ may have non-zero error rate, i.e., the non-realizable (or
agnostic) setting.
In this case, the algorithm \algr, which was designed for the
realizable setting, is no longer applicable.
First, examples obtained by \LABEL and \SEARCH are of different
quality: those returned by \SEARCH always agree with $h^*$, whereas
the labels given by \LABEL need not agree with $h^*$.
Moreover, the version spaces (even when $k=k^*$) as defined by \algr
may always be empty due to the noisy labels.

Another complication arises in our SRM setting that differentiates it
from the usual agnostic active learning setting.  When working with a
specific hypothesis class $H_k$ in the nested sequence, we may observe
high error rates because (i) the finite sample error is too high (but
additional labeled examples could reduce it), or (ii) the current
hypothesis class $H_k$ is impoverished.
In case (ii), the best hypothesis in $H_k$ may have a much larger
error rate than $h^*$, and hence lower bounds~\citep{Kaariainen06}
imply that active learning on $H_k$ instead of $H_{k^*}$ may be
substantially more difficult.

These difficulties in the SRM setting are circumvented by an algorithm that
adaptively estimates the error of $h^*$.
The algorithm,
\alga (Algorithm~\ref{alg:budgetalarch}), is presented in
Appendix~\ref{sec:alarch}.
\begin{theorem}
Assume $\err(h^*) = \nu$. Let $\theta_k(\cdot)$ denote the disagreement coefficient of $V_i^{k_i}$ at
the first iteration $i$ in \alga where $k_i \geq k$. 
Fix any $\epsilon, \delta \in (0,1)$.
If \alga is run with inputs hypothesis classes $\cbr{H_k}_{k=0}^{\infty}$,
  oracles \SEARCH and \LABEL, learning parameter $\delta$, and
  unlabeled example budget $\tilde{O}((d_{k^*} + \log k^*)(\nu +
  \epsilon)/\epsilon^2)$,
then with probability $1-\delta$:
\alga returns a classifier with error rate
  $\leq\nu + \epsilon$; it makes at most $k^* +
  O\del[0]{ \log(d_{k^*}/\epsilon) + \log\log k^* }$ queries to \SEARCH, and
  $\tilde{O}\del[0]{
    \max_{k \leq k^*} \theta_k(2\nu + 2\epsilon) \cdot
    \del[0]{ d_{k^*} \log^2(1/\epsilon) + \log k^* } \cdot
    \del[0]{ 1 + \nu^2/\epsilon^2 }
  }$
  queries to \LABEL.
\label{thm:alarch}
\end{theorem}

The proof is in Appendix~\ref{sec:alarch}.
The \LABEL query
complexity is at least a factor of $k^*$ better
than that in~\cite{Hanneke11}, and sometimes exponentially better
thanks to the reduced disagreement coefficient of the version space
when consistency constraints are incorporated.

\subsection{\algaa: an Opportunistic Anytime Algorithm}
In many practical scenarios, termination conditions based on
quantities like a target excess error rate $\epsilon$ are undesirable.
The target $\epsilon$ is unknown, and we instead prefer an
algorithm that performs as well as possible until a cost budget is
exhausted.
Fortunately, when the primary cost being considered are \LABEL
queries, there are many \LABEL-only active learning algorithms that
readily work in such an ``anytime'' setting~\citep[see, e.g.,][]{DHM07,hanneke15}.

The situation is more complicated when we consider both \SEARCH
and \LABEL: we can often make substantially more progress with
\SEARCH queries than with \LABEL queries (as the error rate of the
best hypothesis in $H_{k'}$ for $k' > k$ can be far lower than in
$H_k$).
\algaa (Algorithm~\ref{alg:aalarch}) shows that
although these queries come at a higher cost, the cost can be
amortized.

\begin{algorithm}[h]
  \caption{\algaa}
  \label{alg:aalarch}

  \begin{algorithmic}[1]

    \REQUIRE Nested hypothesis set $H_0 \subseteq H_1 \subseteq \dotsb$; oracles
    $\LABEL$ and $\SEARCH$; learning parameter $\delta \in (0,1)$;
    $\SEARCH$-to-$\LABEL$ cost ratio $\tau$, dataset size upper bound $N$.

    \ENSURE hypothesis $\tilde{h}$.

    \STATE Initialize: consistency constraints $S \gets \emptyset$, counter $c \gets 0$, $k \gets 0$, verified labeled dataset $\tilde{L} \gets \emptyset$,
    working labeled dataset $L_0 \gets \emptyset$,
    unlabeled examples processed $i \gets 0$, $V_i \gets H_k(S)$.

    \LOOP

          \STATE Reset counter $c \gets 0$.

          \REPEAT
          \label{step:aalarch-start-repeat}

          \IF{$\EC(V_i, L_i, \delta_i)$ }
          \label{step:aalarch-err}

            \STATE $(k, S, V_i) \gets \UVS(k, S, \emptyset)$
            \label{step:aalarch-err-uvs}

            \STATE $V_i \gets \PVS(V_i, \tilde{L}, \delta_i)$

            \STATE $L_i \gets \tilde{L}$

            \STATE \textbf{continue loop}

          \ENDIF

          \STATE $i \gets i + 1$
          \label{step:aalarch-inc-i}

          \STATE $(L_i, c) \gets \SAL(V_{i-1}, \LABEL, L_{i-1}, c)$
          \label{step:aalarch-label-query}

          \STATE $V_i \gets \PVS(V_{i-1}, L_i, \delta_i)$


            \label{step:aalarch-update-version-space}







          \label{alg:algaa:line11}







          \UNTIL{$c = \tau$ \OR $l_i = N$}
          \label{step:aalarch-end-repeat}

          \STATE $e \gets \SEARCH_{H_k}(V_i)$

            \IF{$e \neq \bot$}
            \label{step:aalarch-seed}




              \STATE $(k, S, V_i) \gets \UVS(k, S, \cbr{e})$
              \label{step:aalarch-seed-uvs}

              \STATE $V_i \gets \PVS(V_i, \tilde{L}, \delta_i)$

              \STATE $L_i \gets \tilde{L}$


            \ELSE
                \label{step:aa-verify}

                \STATE Update verified dataset $\tilde{L} \gets L_i$.
                \label{step:aa-storeverified}

                \STATE Store temporary solution $\tilde{h} = \argmin_{h' \in V_i} \err(h', \tilde{L})$.

            \ENDIF

    \ENDLOOP

  \end{algorithmic}

\end{algorithm}

\algaa relies on several subroutines: $\SAL$, $\EC$, $\PVS$ and $\UVS$
(Algorithms~\ref{alg:sal},~\ref{alg:ec}, ~\ref{alg:pvs}, and~\ref{alg:uvs}).
The detailed descriptions are deferred to
Appendix~\ref{appendix-amortized}.
$\SAL$ performs standard disagreement-based selective sampling
using oracle $\LABEL$; labels of examples in the disagreement
region are queried, otherwise inferred. $\PVS$ prunes the version
space given the labeled examples collected, based on standard
generalization error bounds. $\EC$ checks if the best
hypothesis in the version space has large error;
$\SEARCH$ is used to find a systematic mistake for the version space;
if either event happens, \algaa calls $\UVS$ to increase $k$, the level of our working
hypothesis class.


\begin{theorem}
Assume $\err(h^*) = \nu$. Let $\theta_{k'}(\cdot)$ denote the disagreement coefficient of $V_i$ at
the first iteration $i$ after which $k \geq k'$.
Fix any $\epsilon \in (0,1)$. Let
$n_\epsilon = \tilde{O}(\max_{k \leq k^*}\theta_k(2\nu + 2\epsilon) d_{k^*} (1 + \nu^2/\epsilon^2))$
and define $C_\epsilon = 2(n_\epsilon + k^*\tau)$.
Run Algorithm~\ref{alg:aalarch} with a nested sequence of hypotheses $\cbr{H_k}_{k=0}^{\infty}$,
oracles $\LABEL$ and $\SEARCH$,
confidence parameter $\delta$,
cost ratio $\tau \geq 1$,
and upper bound
$N = \tilde{O}(d_{k^*}/\epsilon^2)$.
If the cost spent is at least $C_\epsilon$,
then with probability $1-\delta$, the current hypothesis $\tilde{h}$ has error at most $\nu + \epsilon$.
\label{thm:amortized}
\end{theorem}

The proof is in
Appendix~\ref{appendix-amortized}.
A comparison to
Theorem~\ref{thm:alarch} shows that \algaa is
adaptive: for any cost complexity $C$, the excess error rate $\epsilon$ is
roughly at most twice that achieved by \alga.

\section{Discussion}

The \SEARCH oracle captures a powerful form of interaction
that is useful for machine learning.
Our theoretical analyses of \algr and variants demonstrate that \SEARCH can
substantially improve \LABEL-based active learners, while being
plausibly cheaper to implement than oracles like \CCQ.

Are there examples where \CCQ is substantially more powerful than \SEARCH?  
This is a key question,
because a good active learning system should use minimally powerful
oracles.
Another key question is:
Can the benefits of \SEARCH be provided in a
computationally efficient general purpose manner?


\bibliographystyle{plainnat}
\bibliography{alsearch}

\appendix
\onecolumn

\section{Basic Facts and Notations Used in Proofs}

\subsection{Concentration Inequalities}

\begin{lemma}[Bernstein's Inequality]
Let $X_1, \ldots, X_n$ be independent zero-mean random variables. Suppose that $|X_i| \leq M$ almost surely. Then for all positive $t$,
\[
  \Pr\sbr{\sum_{i=1}^n X_i > t}
  \ \leq \
  \exp\del{-\frac{t^2/2}{\sum_{j=1}^n \E[X_j^2] + Mt/3} }
  \,.
\]
\end{lemma}

\begin{lemma}
Let $Z_1, \ldots, Z_n$ be independent Bernoulli random variables with mean $p$. Let $\bar{Z} = \frac{1}{n}\sum_{i=1}^n Z_i$.
Then with probability $1-\delta$,
\[ \bar{Z} \ \leq \ p + \sqrt{\frac{2p\ln(1/\delta)}{n}} +
\frac{2\ln(1/\delta)}{3n} \,. \]
\label{lem:invbern}
\end{lemma}
\begin{proof}
Let $X_i = Z_i - p$ for all $i$, note that $|X_i| \leq 1$. The lemma follows from Bernstein's Inequality and algebra.
\end{proof}

\begin{lemma}[Freedman's Inequality]
Let $X_1, \ldots, X_n$ be a martingale difference sequence, and $|X_i| \leq M$ almost surely.
Let V be the sum of the conditional variances, i.e.
\[ V = \sum_{i=1}^n \E[X_i^2 | X_1, \ldots. X_{i-1}]\]
Then, for every $t, v > 0$,
\[ \Pr\sbr{\sum_{i=1}^n X_i > t \text{ and } V \leq v} \leq
\exp\del{-\frac{t^2/2}{v + Mt/3} } \,. \]
\end{lemma}

\begin{lemma}
Let $Z_1, \ldots, Z_n$ be a sequence of Bernoulli random variables, where $\E[Z_i|Z_1,\ldots,Z_{i-1}] = p_i$.
Then, for every $\delta > 0$, with probability $1-\delta$:
\[ \sum_{i=1}^n Z_i \leq 2v_n + \sqrt{4v_n \ln\frac{\log 4 n}{\delta}} + \frac{2}{3}\ln\frac{\log 4n}{\delta} \,. \]
where $v_n = \max(\sum_{i=1}^n p_i, 1)$.
\label{lem:invfreed}
\end{lemma}
\begin{proof}
Let $X_i = Z_i - p_i$ for all $i$, note that $\cbr{X_i}$ is a martingale difference sequence and $|X_i| \leq 1$. From Freedman's Inequality and algebra, for any $v$,
\[ \Pr\sbr{\frac{1}{n}\sum_{i=1}^n Z_i > v + \sqrt{\frac{2v \ln\frac{\log 4n}{\delta}}{n}} + \frac{2\ln\frac{\log 4n}{\delta}}{3n} \text{ and } \sum_{i=1}^n p_i \leq v} \leq
\frac{\delta}{\log n + 2} \,. \]
The proof follows by taking union bound over $v = 2^i, i = 0,1,\ldots,\lceil \log n \rceil$.
\end{proof}

Define
\begin{equation}
  \phi(d,m,\delta)
  \ := \
  \frac{1}{m}\del{
    d\log em^2 + \log\frac2\delta
  }
  .
  \label{eq:phi}
\end{equation}

\begin{theorem}[\citealp{VC71}]
  \label{thm:vc}
  Let $\mathcal{F}$ be a family of functions $f \colon \mathcal{Z} \to
  \cbr{0,1}$ on a domain $\mathcal{Z}$ with VC dimension at most $d$, and let
  $P$ be a distribution on $\mathcal{Z}$.
  Let $P_n$ denote the empirical measure from an iid sample of size $n$ from
  $P$.
  For any $\delta \in \intoo{0,1}$, with probability at least $1-\delta$, for
  all $f \in \mathcal{F}$,
  \[
    -\min\cbr{
      \varepsilon
      + \sqrt{ Pf \varepsilon }
      ,\,
      \sqrt{ P_nf \varepsilon }
    }
    \ \leq \
    Pf - P_nf
    \ \leq \
    \min\cbr{
      \varepsilon
      + \sqrt{ P_nf \varepsilon }
      ,\,
      \sqrt{ Pf \varepsilon }
    }
  \]
  where $\varepsilon := \phi(d,n,\delta)$.
\end{theorem}

\subsection{Notations}

For convenience, we define
\begin{equation}
  \sigma(d,m,\delta)
  \ := \
  \phi(d,m,\delta / 3)
  \,,
  \label{eq:sigma}
\end{equation}
as we will often split the overall allowed failure probability $\delta$ across
two or three separate events.

Because we apply the deviation inequalities to the hypothesis classes from
$\cbr{H_k}_{k=0}^\infty$, we also define:
\begin{equation}
  \sigma_k(m,\delta)
  \ := \
  \sigma(d_k,m,\delta)
  ,
  \label{eq:sigma_k}
\end{equation}
where $d_k$ is the VC dimension of $H_k$.
We have the following simple fact.
\begin{fact}
\[
  \sigma\del{
    d, m, \frac{\delta}{2 \log m (\log m + 1)}
  }
  \ \geq \ \epsilon
  \ \Longrightarrow \
  m \ \leq \ \frac{64}{\epsilon} \left(d \log\frac{512}{\epsilon} +
  \log\frac{24}{\delta}\right) \,.
\]
\label{fact:sigma}
\end{fact}



For integers $i \geq 1$ and $k \geq 0$, define
\[
  \delta_i
  :=
  \frac{\delta}{i(i+1)}
  \,,
  \qquad
  \delta_{i,k}
  :=
  \frac{\delta_i}{(k+1)(k+2)}
  \,
\]
Note that $\sum_{i=1}^{\infty} \delta_i = \delta$ and $\sum_{k=0}^{\infty}
\delta_{i,k} = \delta_i$.

Finally, for any distribution $\tilde D$ over $\calX \times \calY$ and any
hypothesis $h \colon \calX \to \calY$, we use $\err(h,\tilde D) := \Pr_{(x,y)
\sim \tilde D}[h(x) \neq y]$ to denote the probability with respect to $\tilde
D$ that $h$ makes a classification error.

\section{Active Learning Algorithm \textsc{CAL}}
\label{sec:cal}

In this section, we describe and analyze a variant of the \LABEL-only
active learning algorithm of \citet{CAL94}, which we refer to as \CAL.
Note that \citet{Hanneke11} provides a label complexity analysis of
\CAL in terms of the disagreement coefficient under the assumption
that the \LABEL oracle is consistent with some hypothesis in the
hypothesis class used by \CAL.
We cannot use that analysis because we call \CAL as a subroutine in
\algr with sets of hypotheses $V$ that do not necessarily contain the
optimal hypothesis $h^*$.

\subsection{Description of \textsc{CAL}}

\CAL takes as input a set of hypotheses $V$, the \LABEL oracle (which
always returns $h^*(x)$ when queried with a point $x$), and learning
parameters $\epsilon, \delta \in (0,1)$.

The pseudocode for \CAL is given in Algorithm~\ref{alg:cal} below,
where we use the notation
\[
  U_{{\leq}i}
  \ := \
  \bigcup_{j=1}^i U_j
\]
for any sequence of sets $(U_j)_{j\in\mathbb{N}}$.

%
%
%

\begin{algorithm}[t]
  \caption{$\CAL$}
  \label{alg:cal}
  \begin{algorithmic}[1]
    \REQUIRE Hypothesis set $V$ with VC dimension ${\leq}d$; oracle $\LABEL$;
    learning parameters $\epsilon,\delta \in (0,1)$

    \ENSURE Labeled examples $T$

    \FOR{$i = 1, 2, \dotsc$}
      \STATE $T_i \gets \emptyset$

      \FOR{$j = 1, 2, \dotsc, 2^i$}
        \STATE $x_{i,j} \gets \text{independent draw from $D_\calX$}$
        (the corresponding label is hidden)

        \IF{$x_{i,j} \in \DIS(V(T_{{\leq}i-1}))$}

          \STATE $T_i \gets T_i \cup \cbr[0]{ (x_{i,j},\LABEL(x_{i,j})) }$

        \ENDIF

      \ENDFOR

      \IF{$\phi(d,2^i,\delta_i/2) \leq \epsilon$ or
      $V(T_{{\leq}i}) = \emptyset$}

        \RETURN $T_{{\leq}i}$

      \ENDIF
    \ENDFOR

  \end{algorithmic}
\end{algorithm}

\subsection{Proof of Lemma~\ref{lem:cal}}

We now give the proof of Lemma~\ref{lem:cal}.

Let $V_0 := V$ and $V_i := V(T_{{\leq}i})$ for each $i \geq 1$.
Clearly $V_0 \supseteq V_1 \supseteq \dotsb$, and hence
$\DIS(V_0) \supseteq \DIS(V_1) \supseteq \dotsb$ as well.

Let $E_i$ be the event in which the following hold:
\begin{enumerate}
  \item
    If \CAL executes iteration $i$, then
    every $h \in V_i$ satisfies
    \[
      \Pr_{x \sim D_{\calX}}
      [ h(x) \neq h^*(x) \ \wedge\ x \in \DIS(V_i) ]
      \ \leq \ \phi(d,2^i,\delta_i/2)
      \,.
    \]

  \item
    If \CAL executes iteration $i$, then
    the number of \LABEL queries in iteration $i$ is at most
    \[
      2^i \mu_i + O\del{
        \sqrt{2^i \mu_i \log(2/\delta_i)} + \log(2/\delta_i)
      }
      \,,
    \]
    where
    \[
      \mu_i \ := \ \theta_{V_{i-1}}(\epsilon) \cdot
      2\phi(d,2^{i-1},\delta_{i-1}/2)
      \,.
    \]

\end{enumerate}
We claim that $E_0 \cap E_1 \cap \dotsb \cap E_i$ holds with probability at
least $1-\sum_{i'=1}^i\delta_{i'} \geq 1-\delta$.
The proof is by induction.
The base case is trivial, as $E_0$ holds deterministically.
For the inductive case, we just have to show that $\Pr(E_i \mid E_0
\cap E_1 \cap \dotsb \cap E_{i-1}) \geq 1 - \delta_i$.

Condition on the event $E_0 \cap E_1 \cap \dotsb \cap E_{i-1}$.
Suppose \CAL executes iteration $i$.
For all $x \notin \DIS(V_{i-1})$, let $V_{i-1}(x)$ denote the label
assigned by every $h \in V_{i-1}$ to $x$.
Define
\[
  \hat S_i
  \ := \
  \cbr{
    (x_{i,j},\hat y_{i,j}) :
    j \in \cbr[0]{1,2,\dotsc,2^i} ,\
    x_{i,j} \notin \DIS(V_{i-1}) ,\
    \hat y_{i,j} = V_{i-1}(x_{i,j})
  }
  \,.
\]
Observe that $\hat S_i \cup T_i$ is an iid sample of size $2^i$ from a
distribution (which we call $D_{i-1}$) over labeled examples $(x,y)$, where
$x \sim D_\calX$ and $y$ is given by
\[
  y \ := \
  \begin{cases}
    V_{i-1}(x) & \text{if $x \notin \DIS(V_{i-1})$} \,, \\
    h^*(x) & \text{if $x \in \DIS(V_{i-1})$} \,.
  \end{cases}
\]
In fact, for any $h \in V_{i-1}$, we have
\begin{equation}
  \err_{D_{i-1}}(h)
  \ = \
  \Pr_{(x,y) \sim D_{i-1}}[ h(x) \neq y ]
  \ = \
  \Pr_{x \sim D_{\calX}}
  [ h(x) \neq h^*(x) \ \wedge \ x \in \DIS(V_{i-1}) ]
  \,.
  \label{eq:err-modified}
\end{equation}
The VC inequality (Theorem~\ref{thm:vc})
implies that,
with probability at least $1-\delta_i/2$,
\begin{equation}
  \forall h \in V
  \centerdot
  \del{
    \err(h,\hat S_i \cup T_i) \ = \ 0
    \ \ \Longrightarrow\ \
    \err_{D_{i-1}}(h) \ \leq \ \phi(d,2^i,\delta_i/2)
  }
  \,.
  \label{eq:cal-vc}
\end{equation}
Consider any $h \in V_i$.
We have $\err(h,T_i) = 0$ by definition of $V_i$.
We also have $\err(h,\hat S_i) = 0$ since $h \in V_i \subseteq
V_{i-1}$.
So in the event that~\eqref{eq:cal-vc} holds, we have
\begin{align*}
  \Pr_{x \sim D_{\calX}}
  [ h(x) \neq h^*(x) \ \wedge \ x \in \DIS(V_i) ]
  & \ \leq \
  \Pr_{x \sim D_{\calX}}
  [ h(x) \neq h^*(x) \ \wedge \ x \in \DIS(V_{i-1}) ]
  \\
  & \ = \
  \err_{D_{i-1}}(h) 
  \\
  & \ \leq \
  \phi(d,2^i,\delta_i/2)
  \,,
\end{align*}
where the first inequality follows because $\DIS(V_i) \subseteq
\DIS(V_{i-1})$, and the equality follows
from~\eqref{eq:err-modified}.

Now we prove the \LABEL query bound.
\begin{claim}
  On event $E_{i-1}$ for every $h,h' \in V_{i-1}$,
  \[
      \Pr_{x \sim D_{\calX}}[ h(x) \neq h'(x) ] \leq \
      2\phi(d,2^{i-1},\delta_{i-1}/2)
  \]
\end{claim}
\begin{proof}
  On event $E_{i-1}$, every $h \in V_{i-1}$ satisfies
  \[
    \Pr_{x \sim D_{\calX}}[ h(x) \neq h^*(x) ,\, x \in \DIS(V_{i-1}) ]
    \ \leq \ \phi(d,2^{i-1},\delta_{i-1}/2)
    \,.
  \]
  Therefore, for any $h, h' \in V_{i-1}$, we have
  \begin{align*}
    \Pr_{x \sim D_{\calX}}[ h(x) \neq h'(x) ]
    & \ = \
    \Pr_{x \sim D_{\calX}}[ h(x) \neq h'(x) ,\, x \in \DIS(V_{i-1}) ]
    \\
    & \ \leq \
    \Pr_{x \sim D_{\calX}}[ h(x) \neq h^*(x) ,\, x \in \DIS(V_{i-1}) ]
    \\
    & \qquad
    + \Pr_{x \sim D_{\calX}}[ h'(x) \neq h^*(x) ,\, x \in \DIS(V_{i-1}) ]
    \\
    & \ \leq \
    2\phi(d,2^{i-1},\delta_{i-1}/2)
    \,.
    \qedhere
  \end{align*}
\end{proof}
Since \CAL does not halt before iteration $i$, we have
$2\phi(d,2^{i-1},\delta_{i-1}/2) \geq \epsilon$, and hence the above claim and
the definition of the disagreement coefficient imply
\[
  \Pr_{x \sim D_{\calX}}[ x \in \DIS(V_{i-1}) ]
  \ \leq \
  \theta_{V_{i-1}}(\epsilon)
  \cdot
  2\phi(d,2^{i-1},\delta_{i-1}/2)
  \ = \
  \mu_i
  \,.
\]
Therefore, $\mu_i$ is an upper bound on the probability that \LABEL
is queried on $x_{i,j}$, for each $j = 1,2,\dotsc,2^i$.
By Lemma~\ref{lem:invbern}, the number of queries to \LABEL is at most
\[
  2^i \mu_i
  +
  O\del{
    \sqrt{
      2^i \mu_i \log(2/\delta_i)
    }
    + \log(2/\delta_i)
  }
  \,.
\]
with probability at least $1-\delta_i/2$.
We conclude by a union bound that $\Pr(E_i \mid E_0 \cap E_1 \cap
\dotsb \cap E_{i-1}) \geq 1 - \delta_i$ as required.

We now show that in the event $E_0 \cap E_1 \cap \dotsb$, which holds with
probability at least $1-\delta$, the required consequences from
Lemma~\ref{lem:cal} are satisfied.
The definition of $\phi$ from~\eqref{eq:phi} and the halting condition in
\CAL imply that the number of iterations $I$ executed by \CAL satisfies
\[ 
  \sigma(d,2^{I-1},\delta_{I-1}/2)
  \ \geq \
  \epsilon
  \,.
\]
Thus by Fact~\ref{fact:sigma},
\[ 
  2^I
  \ \leq \
  O\del{
    \frac{1}{\epsilon}
    \del{ d\log \frac{1}{\epsilon} + \log \frac{1}{\delta} }
  }
  \,,
\]
which immediately gives the required bound on the number of unlabeled points
drawn from $D_\calX$.
Moreover, $I$ can be bounded as
\[
  I
  \ = \
  O\del{
    \log(d/\epsilon)
    + \log\log(1/\delta)
  }
  \,.
\]
Therefore, in the event $E_0 \cap E_1 \cap \dotsb \cap E_I$, \CAL
returns a set of labeled examples $T := T_{{\leq}I}$ in which every $h
\in V(T)$ satisfies
\[
  \Pr_{x \sim D_{\calX}}
  [ h(x) \neq h^*(x) \ \wedge \ x \in \DIS(V(T)) ]
  \ \leq \
  \epsilon
  \,,
\]
and the number of \LABEL queries is bounded by
\begin{align*}
  \lefteqn{
    \sum_{i=1}^I
    \del{
      2^i \mu_i
      +
      O\del{
        \sqrt{
          2^i \mu_i \log(2/\delta_i)
        }
        + \log(2/\delta_i)
      }
    }
  }
  \\
  & \ = \
  \sum_{i=1}^I
  O\del{
    2^i 
    \cdot
    \del{
       \theta_{V_{i-1}}(\epsilon)
       \frac{d \log 2^i 
       + \log(2/\delta_i)}
       {2^i}
    }
    +\log(2/\delta_i)
  }
  \\
  & \ = \
  \sum_{i=1}^I
  O\del{
    \theta_{V_{i-1}}(\epsilon)
    \cdot
    \del{
      d \cdot i + \log(1/\delta)
    }
  }
  \\
  & \ = \
  O\del{
    \theta_V(\epsilon)
    \cdot
    \del{
      d
      \cdot \del{ \log(d/\epsilon) + \log\log(1/\delta) }^2
      +
      \del{ \log(d/\epsilon) + \log\log(1/\delta) }
      \cdot \log(1/\delta)
    }
  }
  \\
  & \ = \
  \tilde{O}
  \del{
    \theta_V(\epsilon)
    \cdot
    d
    \cdot
    \log^2(1/\epsilon)
  }
\end{align*}
as claimed.
\qed

\section{An Improved Algorithm for the Realizable Case}
\label{sec:seabel}

In this section, we present an improved algorithm for using \SEARCH and \LABEL
in the realizable section.
We call this algorithm \algrb (Algorithm~\ref{alg:seabel}).

\subsection{Description of \algrb}

\algrb proceeds in iterations like \algr, but takes more advantage of \SEARCH.
Each iteration is split into two stages: the verification stage, and the
sampling stage.

In the verification stage
(Steps~\ref{step:start-verification-seabel}--\ref{step:end-verification-seabel}),
\algrb makes repeated calls to \SEARCH to advance to as high of a complexity
class as possible, until $\bot$ is returned.
When $\bot$ is returned, it guarantees that whenever the latest version space
completely agrees on an unlabeled point, then it is also in agreement with
$h^*$, even if it does not contain $h^*$.

In the sampling stage
(Steps~\ref{step:start-sampling-seabel}--\ref{step:end-sampling-seabel}), \algrb
performs selective sampling, querying and infering labels based on disagreement
over the new version space $V_i^{k_i}$.
The preceding verification stage ensures that whenever a label is inferred, it
is guaranteed to be in agreement with $h^*$.

The algorithm calls Algorithm~\ref{alg:sal} in Appendix~\ref{appendix-amortized},
 where we slightly abuse the
notation in \SAL that if the counter parameter
is missing then it simply does not get updated.

\begin{algorithm}[t]
  \caption{\algrb}
  \label{alg:seabel}

  \begin{algorithmic}[1]

    \REQUIRE Nested hypothesis classes $H_0 \subseteq H_1 \subseteq \dotsb$; oracles
    $\SEARCH$ and $\LABEL$; learning parameters $\epsilon,\delta \in (0,1)$

    \STATE \textbf{initialize} $S_0 \gets \emptyset$, $k_0 \gets 0$.

    \STATE Draw $x_{1,1}, x_{1,2}$ at random from $D_\calX$, $T_1 \gets \cbr{(x_{1,1},\LABEL(x_{1,1})), (x_{1,2},\LABEL(x_{1,2}))}$

    \FOR{iteration $i = 1, 2, \dotsc$}

          \STATE
          $S \gets S_{i-1}$,
          $k \gets \min\cbr{k' \geq k_{i-1} : H_{k'}(S_{i-1} \cup T_i) \neq \emptyset}$
          \COMMENT{Verification stage
          (Steps~\ref{step:start-verification-seabel}--\ref{step:end-verification-seabel})}
          \label{step:start-verification-seabel}

          \LOOP

                  \STATE $e \gets \SEARCH_{H_k}(H_k(S \cup T_i))$

                  \IF{$e \neq \bot$}

                    \STATE $S \gets S \cup \cbr{e}$

                    \STATE $k \gets \min\cbr{k' > k: H_{k'}(S \cup T_i) \neq \emptyset}$


                  \ELSE

                    \STATE \textbf{break}

                  \ENDIF

          \ENDLOOP
          \label{step:end-verification-seabel}

          \STATE $S_i \gets S$, $k_i \gets k$
          \COMMENT{Sampling stage
          (Steps~\ref{step:start-sampling-seabel}--\ref{step:end-sampling-seabel})}
          \label{step:start-sampling-seabel}

          \STATE Define new version space $V_{i}^{k_i} = H_{k_i}(S_i \cup T_i)$

          \STATE $T_{i+1} \gets \emptyset$

          \FOR{$j = 1, 2, \dotsc, 2^{i+1}$}

            \STATE $T_{i+1} \gets \SAL(V_i^{k_i}, \LABEL, T_{i+1})$






          \ENDFOR
          \label{step:end-sampling-seabel}

          \IF{$\sigma_{k_i}(2^{i},\delta_{i,k_i}) \leq \epsilon$}
            \RETURN any $\hat{h} \in V_i^{k_i}$
          \ENDIF

    \ENDFOR

  \end{algorithmic}

\end{algorithm}

\subsection{Proof of Theorem~\ref{thm:seabel}}

Observe that $T_{i+1}$ is an iid sample of size $2^{i+1}$ from a distribution
(which we call $D_i$) over
labeled examples $(x,y)$, where $x \sim D_\calX$, and
\[
  y
  \ := \
  \begin{cases}
    V_i^{k_i}(x)
    & \text{if $x \notin \DIS(V_i^{k_i})$} \,, \\
    h^*(x)
    & \text{if $x \in \DIS(V_i^{k_i})$} \,,
  \end{cases}
\]
for every $x$ in the support of $D_\calX$.
($T_1$ is an iid sample from $D_0 := D$; also note $k_0 = 0$ and $S_0 = \emptyset$.)

\begin{lemma}
  \label{lem:seabel}
  Algorithm~\ref{alg:seabel} maintains the following invariants:
  \begin{enumerate}
    \item The loop in the verification stage of iteration $i$ terminates for all $i\geq1$.
    \item $k_i \leq k^*$ for all $i\geq0$.
    \item $h^*(x) = V_i^{k_i}(x)$ for all $x \notin \DIS(V_i^{k_i})$ for all $i\geq1$.
    \item $h^*$ is consistent with $S_i \cup T_{i+1}$ for all $i\geq0$.
  \end{enumerate}
\end{lemma}

\begin{proof}
It is easy to see that $S$ only contains examples provided by \SEARCH, and hence the labels are
consistent with $h^*$.

Now we prove that the invariants hold by induction on $i$, starting with $i=0$.
For the base case, only the last invariant needs checking, and it is true because the labels in
$T_1$ are obtained from \LABEL.

For the inductive step, fix any $i\geq1$, and assume that $k_{i-1} \leq k^*$, and that $h^*$ is
consistent with $T_i$.
Now consider the verification stage in iteration $i$.
We first prove that the loop in the verification stage will terminate and establish some properties
upon termination.
Observe that $k$ and $S$ are initially $k_{i-1}$ and $S_{i-1}$, respectively.
Throughout the loop, the examples added to $S$ are obtained from \SEARCH, and hence are consistent
with $h^*$.
Thus, $h^* \in H_{k^*}(S \cup T_i)$, implying $H_{k^*}(S \cup T_i) \neq \emptyset$.
If $k = k^*$, then $\SEARCH_{H_{k^*}}(H_{k^*}(S \cup T_i))$ would return $\bot$ and
Algorithm~\ref{alg:seabel} would exit the loop.
If $\SEARCH_{H_k}(H_k(S \cup T_i)) \neq \bot$, then $k < k^*$, and $k$ cannot be increased beyond
$k^*$ since $H_{k^*}(S \cup T_i) \neq \emptyset$.
Thus, the loop must terminate with $k \leq k^*$, implying $k_i \leq k^*$.
This establishes invariants 1 and 2.
Moreover, because the loop terminates with $\SEARCH_{H_k}(H_k(S \cup T_i))$ returning $\bot$ (and
here, $k = k_i$ and $H_k(S \cup T_i) = V_i^{k_i}$), there is no counterexample $x \in \calX$ such
that $h^*$ disagrees with every $h \in V_i^{k_i}$.
This implies that $h^*(x) = V_i^{k_i}(x)$ for all $x \notin \DIS(V_i^{k_i})$, i.e., invariant 3.

Now consider any $(x,y)$ added to $T_{i+1}$ in the sampling stage.
If $x \in \DIS(V_i^{k_i})$, the label is obtained from \LABEL, and hence is consistent with $h^*$;
if $x \notin \DIS(V_i^{k_i})$, the label is $V_i^{k_i}(x)$, which is the same as $h^*(x)$ as
previously argued.
So $h^*$ is consistent with all examples in $T_{i+1}$, and hence also all examples in $S_i \cup
T_{i+1}$, proving invariant 4. This completes the induction.
%
\end{proof}

Let $E_i$ be the event in which the following hold:
\begin{enumerate}
  \item
    For every $k \geq 0$, every $h \in H_k$ satisfies
    \[
      \err(h, D_{i-1})
      \ \leq \
      \err(h, T_i)
      +
      \sqrt{ \err(h, T_i)
      \sigma_k(2^i,\delta_{i,k})}
      +
      \sigma_k(2^i,\delta_{i,k})
      \,.
    \]

  \item The number of $\LABEL$ queries in iteration $i$ (to form $T_{i+1}$) is at most
    \[
      2^{i+1} \Pr_{x \sim D_\calX}[x \in \DIS(V_i^{k_i})] + O\del{
        \sqrt{2^{i+1} \Pr_{x \sim D_\calX}[x \in \DIS(V_i^{k_i})] \log(1/\delta_i)} + \log(1/\delta_i)
      }
      \,,
    \]
\end{enumerate}
Using Theorem~\ref{thm:vc} and Lemma~\ref{lem:invbern}, along with the union bound, $\Pr(E_i) \geq 1 -
\delta_i$.
Define $E := \cap_{i=1}^{\infty} E_i$; a union bound implies that $\Pr(E) \geq 1 - \delta$.

We now prove Theorem~\ref{thm:seabel}, starting with the error rate guarantee.
Condition on the event $E$.
Since $k_i \leq k^*$ (Lemma~\ref{lem:seabel}), the definition of $\sigma_k$ from~\eqref{eq:sigma_k}, the
halting condition in Algorithm~\ref{alg:seabel}, and Fact~\ref{fact:sigma} imply that the algorithm
must halt after at most $I$ iterations, where
\begin{equation}
  2^I
  \ \leq \
  O\del{
    \frac{1}{\epsilon}
    \del{ d_{k^*} \log \frac{1}{\epsilon} + \log \frac{k^*}{\delta} }
  }
  \,.
  \label{eq:seabel-iter}
\end{equation}
So let $I$ denote the iteration in which Algorithm~\ref{alg:seabel} halts.
By definition of $E_I$, we have
\begin{align*}
  \err(\hat{h}, D_{i-1})
  & \ \leq \ \err(\hat{h}, T_{i}) + \sqrt{\err(\hat{h}, T_{i}) \sigma_{k_i}(2^{i}, \delta_{i,{k_i}})} +
  \sigma_{k_i}(2^{i}, \delta_{i,{k_i}}) \\
  & \ = \ \sigma_{k_i}(2^{i}, \delta_{i,{k_i}})
  \ \leq \ \epsilon
  \,,
\end{align*}
where the second inequality follows from the termination condition.
By Lemma~\ref{lem:seabel}, $h^*(x) = V_{i-1}^{k_{i-1}}(x)$ for all $x \notin \DIS(V_{i-1}^{k_{i-1}})$.
Therefore, $D(\cdot \mid x) = D_{i-1}(\cdot \mid x)$ for every $x$ in the support of $D_\calX$, and
\[
  \err(\hat{h},D)
  \ = \ \err(\hat{h},D_{i-1})
  \ \leq \ \epsilon
  \,.
\]

Now we bound the unlabeled, \LABEL, and \SEARCH complexities, all conditioned on
event $E$.
First, as argued above, the algorithm halts after at most $I$ iterations, where
$2^I$ is bounded as in~\eqref{eq:seabel-iter}.
The number of unlabeled examples drawn from $D_\calX$ across all iterations is
within a factor of two of the number of examples drawn in the final sampling
stage, which is $O(2^I)$.
Thus~\eqref{eq:seabel-iter} also gives the bound on the number of unlabeled
examples drawn.

Next, we consider the \SEARCH complexity.
For each iteration $i$, each call to $\SEARCH$ either returns a counterexample
that forces $k$ to increment (but never past $k^*$, as implied by
Lemma~\ref{lem:seabel}), or returns $\bot$ which causes an exit from the
verification stage loop.
Therefore, the total number of $\SEARCH$ calls is at most
\[
  k^* + I
  \ = \
  k^* + O\del{ \log\frac{d_{k^*}}{\epsilon} + \log\log\frac{k^*}{\delta} }
  .
\]

Finally, we consider the \LABEL complexity.
For $i \leq I$, we first show that the version space $V_i^{k_i}$ is always contained in a ball of
small radius (with respect to the disagreement pseudometric).
Specifically, for every $h, h'$ in $V_i^{k_i}$, $\err(h, T_i) = 0$ and $\err(h, T_i) = 0$.
By definition of $E_i$, this implies that
\[ \err(h, D_{i-1}) \ \leq \ \sigma_{k_i}(2^i,\delta_{i,k_i})
\quad \text{and} \quad
 \err(h', D_{i-1}) \ \leq \ \sigma_{k_i}(2^i,\delta_{i,k_i}) . \]
Therefore, by the triangle inequality and the fact $k_i \leq k^*$,
\[
  \Pr_{x\sim D}[h(x) \neq h'(x)]
  \ \leq \
  2\sigma_{k_i}(2^i,\delta_{i,k_i})
  \ \leq \
  2\sigma_{k^*}(2^i,\delta_{i,k^*})
  .
\]
Also, the upper bound $2^I \leq \tilde{O}(d_{k^*}/\epsilon)$
from~\eqref{eq:seabel-iter} implies the lower bound $\sigma_{k^*}(2^i,\delta_{i,
k^*}) \geq \epsilon/2$ for $i \leq I$.
Thus, the probability mass of the disagreement region can be bounded as
\begin{eqnarray*}
  \Pr_{x \sim D_\calX}[x \in \DIS(V_i^{k_i})]
  &\leq& \theta_{k_i}(\epsilon) \cdot 2\sigma_{k^*}(2^i, \delta_{i, k^*})
  .
\end{eqnarray*}
By definition of $E_i$, the number of queries to $\LABEL$ at iteration $i$ is at most
\[
  2^{i+1} \Pr_{x \sim D_\calX}[x \in \DIS(V_i^{k_i})] + O\del{\sqrt{2^{i+1} \Pr_{x \sim D_\calX}[x
  \in \DIS(V_i^{k_i})] \log(1/\delta_{i, k})} + \log(1/\delta_{i})}
  ,
\]
which is at most
\[
  O\del{2^i \cdot \theta_{k_i}(\epsilon) \cdot \sigma_{k^*}(2^i, \delta_{i, k^*}) }
  .
\]
We conclude that the total number of $\LABEL$ queries by Algorithm~\ref{alg:seabel} is bounded by
\begin{align*}
  \lefteqn{
    2 + \sum_{i=1}^I O\del{2^i \cdot \theta_{k_i}(\epsilon) \cdot \sigma_{k^*}(2^i, \delta_{i, k^*})}
  } \\
  & \ = \ 2 + \sum_{i=1}^I O\del{2^i \cdot \max_{k \leq k^*} \theta_k(\epsilon) \cdot \sigma_{k^*}(2^i, \delta_{i, k^*}) } \\
  & \ = \ O\del{\max_{k \leq k^*}\theta_k(\epsilon) \cdot \left( \sum_{i=1}^I 2^i \cdot \frac{d\ln(2^i) + \ln(\frac{(i^2+i)(k^*)^2}{\delta})}{2^i} \right) } \\
  & \ = \ O\del{\max_{k \leq k^*}\theta_k(\epsilon) \cdot \left( d_{k^*} I^2 + I \log\frac{k^*}{\delta} \right) }\\
  & \ = \ O\del{\max_{k \leq k^*}\theta_k(\epsilon) \cdot \del{ d_{k^*} \del{
\log\frac{d_{k^*}}{\epsilon} + \log\log\frac{k^*}{\delta} }^2 + \del{ \log\frac{d_{k^*}}{\epsilon} +
\log\log\frac{k^*}{\delta} } \log\frac{k^*}{\delta} } } \\
& \ = \ \tilde{O}\del{\max_{k \leq k^*}\theta_k(\epsilon) \cdot \del{ d_{k^*} \cdot
\log^2\frac{1}{\epsilon} + \log k^*}}
\end{align*}
as claimed.
\qed

\section{\alga: An Adaptive Agnostic Algorithm}
\label{sec:alarch}

In this section, we present a generalization of \algrb that works in the
agnostic setting.
We call this algorithm \alga (Algorithm~\ref{alg:budgetalarch}).


\subsection{Description of \alga}

\alga proceeds in iterations like \algrb.
Each iteration is split into three stages: the error estimation stage,
the verification stage, and the
sampling stage.

In the error estimation stage, $\alga$ uses a structural risk minimization
approach (Step~\ref{step:srm-alarch})
to
compute $\gamma_{i-1}$, a (tight) upper bound on $\Pr[h^*(x) \neq y, x \in
\DIS(V_{i-1})]$. (See item 1 of Lemma~\ref{lem:invkagnostic} for justification.)

The verification stage
(Steps~\ref{step:start-verification-alarch}--\ref{step:end-verification-alarch})
and sampling stage
(Steps~\ref{step:start-sampling-alarch}--\ref{step:end-sampling-alarch}) in
\alga are similar to the corresponding stages in \algrb.

Same as \algrb, the algorithm calls Algorithm~\ref{alg:sal},~\ref{alg:pvs}
and~\ref{alg:uvs} in Appendix~\ref{appendix-amortized} (\SAL, \PVS and \UVS,
respectively), where we slightly abuse the notation in \SAL that if the counter parameter
is missing then it simply does not get updated.


%
%
%

\begin{algorithm}[H]
  \caption{\alga}
  \label{alg:budgetalarch}

  \begin{algorithmic}[1]

    \REQUIRE Nested hypothesis set $H_0 \subseteq H_1 \subseteq \dotsb$; oracles $\LABEL$ and $\SEARCH$; learning parameter $\delta \in (0,1)$; unlabeled examples budget $m = 2^{I+2}$.

    \ENSURE hypothesis $\hat{h}$.

    \STATE \textbf{initialize} $S \gets \emptyset$, $k_0\gets 0$.

    \STATE Draw $x_{1,1}, x_{1,2}$ at random from $D_\calX$, $T_1 \gets \cbr{(x_{1,1},\LABEL(x_{1,1})), (x_{1,2},\LABEL(x_{1,2}))}$

    \FOR{$i = 1, 2, \dotsc, I$}

          \STATE $\gamma_{i-1} \gets \min_{k' \geq k_{i-1}, h \in H_{k'}}\cbr{\err(h, T_i) + \sqrt{\err(h, T_i)\sigma_{k'}(2^i, \delta_{i,k'})} + \sigma_{k'}(2^i, \delta_{i,k'})}$
          \label{step:srm-alarch}
          \COMMENT{Error estimation stage (Step~\ref{step:srm-alarch})}



          \STATE $S \gets S_{i-1}$, $k \gets k_{i-1}$.
          \COMMENT{Verification stage (Steps~\ref{step:start-verification-alarch}--\ref{step:end-verification-alarch})}
          \label{step:start-verification-alarch}

          \STATE $V_i^k \gets \PVS(H_k(S), T_i, \delta_i)$



          \LOOP

              \IF{$\min_{h \in H_k(S)} \err(h, T_i) > \gamma_{i-1} + \sqrt{\gamma_{i-1}\sigma_k(2^i, \delta_{i,k})} + \sigma_k(2^i, \delta_{i,k})$}
              \label{step:check-error-alarch}

                  \STATE $(k, S, V_i^k) \gets \UVS(k, S, \emptyset)$

              \ELSE


                  \STATE $e \gets \SEARCH_{H_k}(V_{i}^k)$

                  \IF{$e \neq \bot$}

                    \STATE $(k, S, V_i^k) \gets \UVS(k, S, \cbr{e})$



                  \ELSE

                    \STATE \textbf{break}

                  \ENDIF

              \ENDIF

          \ENDLOOP
          \label{step:end-verification-alarch}

          \STATE $S_i \gets S$, $k_i \gets k$

          \STATE $T_{i+1} \gets \emptyset$ \COMMENT{Sampling stage (Steps~\ref{step:start-sampling-alarch}--\ref{step:end-sampling-alarch})}
          \label{step:start-sampling-alarch}

          \FOR{$j = 1, 2, \dotsc, 2^{i+1}$}







            \STATE $T_{i+1} \gets \SAL(V_i^{k_i}, \LABEL, T_{i+1})$

          \ENDFOR
          \label{step:end-sampling-alarch}

    \ENDFOR

    \RETURN any $\hat{h} \in V_I^{k_I}$.

  \end{algorithmic}

\end{algorithm}

\subsection{Proof of Theorem~\ref{thm:alarch}}
Let
\begin{eqnarray*}
  M(\nu, k^*,\epsilon,\delta)
  \ &:=& \
  \min\cbr{ 2^n : n \in \mathbb{N} ,\, 6\sqrt{\nu \sigma_{k^*}(2^n,
  \delta_{n,{k^*}})} + 21\sigma_{k^*}(2^n, \delta_{n,{k^*}}) \leq \epsilon }
  \\
  \ &=& \
  O\del{
    \frac{
      (d_{k^*}\log(1/\epsilon) + \log(k^*/\delta))
      (\nu + \epsilon)
    }{\epsilon^2}
  }
  .
\end{eqnarray*}
where the second line is from Fact~\ref{fact:sigma}.



\begin{theorem}[Restatement of Theorem~\ref{thm:alarch}]
Assume $\err(h^*) = \nu$. If Algorithm~\ref{alg:budgetalarch} is run with inputs hypothesis classes $\cbr{H_k}_{k=0}^{\infty}$, oracles $\SEARCH$ and $\LABEL$, learning parameter $\delta$, unlabeled examples budget $m = M(\nu, k^*,\epsilon,\delta)$ and the disagreement coefficient of $H_k(S)$ is at most $\theta_k(\cdot)$, then, with probability $1-\delta$:\\
(1) The returned hypothesis $\hat{h}$ satisfies
\[ \err(\hat{h}) \leq \nu + \epsilon~. \]
(2) The total number of queries to oracle $\SEARCH$ is at most
\[ k^* + \log m \leq k^* + O\del{\log\frac{d_{k^*}}{\epsilon} + \log\log\frac{k^*}{\delta}}. \]
(3) The total number of queries to oracle $\LABEL$ is at most
\[ \tilde{O}\left(\max_{k \leq k^*} \theta_k(2\nu + 2\epsilon) \cdot d_{k^*}\left(\log \frac{1}{\epsilon}\right)^2 \cdot \left(1 + \frac{\nu^2}{\epsilon^2}\right)\right)~. \]
\label{thm:restatealarch}
\end{theorem}

The proof relies on an auxillary lemma.
First, we need to introduce the following notation.


Observe that $T_{i+1}$ is an iid sample of size $2^{i+1}$ from a
distribution (which we call $D_i$) over labeled examples $(x,y)$,
where $x \sim D_\calX$ and the conditional distribution is
\[
  D_{i}(y\mid x) \ := \
  \begin{cases}
    \ind{y = V_{i}^{k_{i}}(x)} & \text{if $x \notin \DIS(V_{i}^{k_{i}})$} \,, \\
    D(y\mid x) & \text{if $x \in \DIS(V_{i}^{k_{i}})$} \,.
  \end{cases}
\]
$T_1$ is a sample of size 2 from $D_0 := D$.

Let $E_i$ be the event in which the following hold:
\begin{enumerate}
  \item
    For every $k \geq 0$, every $h \in H_k$ satisfies
    \begin{align*}
      \err(h, D_{i-1})
      & \ \leq \
      \err(h, T_i)
      +
      \sqrt{ \err(h, T_i)
      \sigma_k(2^i,\delta_{i,k})}
      +
      \sigma_k(2^i,\delta_{i,k})
      \,,
      \\
      \err(h, T_i)
      & \ \leq \
      \err(h, D_{i-1})
      +
      \sqrt{ \err(h, D_{i-1})
      \sigma_k(2^i,\delta_{i,k})}
      +
      \sigma_k(2^i,\delta_{i,k})
      \,.
    \end{align*}

  \item The number of $\LABEL$ queries at iteration $i$ is at most
    \[
      2^{i+1} \Pr_{x \sim D_\calX}[x \in \DIS(V_i^{k_i})] + O\del{
        \sqrt{2^{i+1} \Pr_{x \sim D_\calX}[x \in \DIS(V_i^{k_i})] \log(1/\delta_i)} + \log(1/\delta_i)
      }
      \,.
    \]
\end{enumerate}

Using Theorem~\ref{thm:vc} and Lemma~\ref{lem:invbern}, along with the union bound, $\Pr(E_i) \geq 1 - \delta_i$.
Define $E: =\cap_{i=1}^{\infty} E_i$, by union bound, $\Pr(E) \geq 1 - \delta$. Recall that $k_i$ is the value of $k$ at the end of iteration $i$.

\begin{lemma}\label{lemma:alarch-invariants}
On event $E$, Algorithm~\ref{alg:budgetalarch} maintains the following invariants:
\begin{enumerate}
\item For all $i \geq 1$, $\gamma_{i-1}$ is such that
\[ \err(h^*, D_{i-1}) \leq \gamma_{i-1} \leq \err(h^*, D_{i-1}) + 2\sqrt{\err(h^*, D_{i-1}) \sigma_{k^*}(2^i, \delta_{i,k^*})} + 3\sigma_{k^*}(2^i, \delta_{i,k^*}).\]
\item The loop in the verification stage of iteration $i$ terminates for all $i\geq1$.
\item $k_i \leq k^*$ for all $i\geq0$.
\item $h^*(x) = V_i^{k_i}(x)$ for all $x \notin \DIS(V_i^{k_i})$ for all $i\geq1$.
\item For all $i\geq0$, for every hypothesis $h$, $\err(h, D_i) - \err(h^*, D_i) \geq \err(h, D) - \err(h^*, D)$. Therefore, $h^*$ is the optimal hypothesis among $\cup_{k} H_k$ with respect to $D_i$.
\end{enumerate}
\label{lem:invkagnostic}
\end{lemma}
\begin{proof}
  Throughout, we assume the event $E$ holds.

It is easy to see that $S$ only contains examples provided by \SEARCH, and hence the labels are
consistent with $h^*$.

Now we prove that the invariants hold by induction on $i$, starting with $i=0$. For the base case, invariant 3 holds since $k_0 = 0 \leq k^*$, and invariant 5 holds since $D_0 = D$ and $h^*$ is the optimal hypothesis in $\cup_k H_k$.

Now consider the inductive step.
We first prove that invariant 1 holds.
  \begin{enumerate}
    \item[(1)]
By definition of $E_i$, for all $k' \geq k_{i-1}$, we have for all $h \in H_{k'}$,
\[ \err(h, D_{i-1}) \leq \err(h, T_i) + \sqrt{\err(h, T_i)\sigma_{k'}(2^i,
  \delta_{i,k'})} + \sigma_{k'}(2^i, \delta_{i,k'}) \,. \]
Thus,
\[ \min_{h \in H_{k'}} \err(h, D_{i-1}) \leq \min_{h \in H_{k'}}\err(h, T_i) +
  \sqrt{\err(h, T_i)\sigma_{k'}(2^i, \delta_{i,k'})} + \sigma_{k'}(2^i,
  \delta_{i,k'}) \,. \]
Taking minimum over $k' \geq k_{i-1}$ on both sides, notice that $h^*$ is the optimal hypothesis with respect to $D_{i-1}$ and recall the definition of $\gamma_{i-1}$, we get
\[ \err(h^*, D_{i-1}) \leq \gamma_{i-1}  \,. \]

      \item[(2)]
By definition of $\gamma_{i-1}$, we have
\[ \gamma_{i-1} = \min_{k' \geq k_{i-1}, h \in H_{k'}}\cbr{\err(h, T_i) + \sqrt{\err(h, T_i)\sigma_{k'}(2^i, \delta_{i,k'})} + \sigma_{k'}(2^i, \delta_{i,k'})} \]
Taking $k' = k^*$, $h = h^*$, we get
\[ \gamma_{i-1} \leq \err(h^*, T_i) + \sqrt{\err(h^*, T_i)\sigma_{k^*}(2^i, \delta_{i,k^*})} + \sigma_{k^*}(2^i, \delta_{i,k^*}) \]
In conjunction with the fact that by definition of $E_i$,
\[ \err(h^*, T_i) \leq \err(h^*, D_{i-1}) + \sqrt{\err(h^*, D_{i-1}) \sigma_{k^*}(2^i, \delta_{i,k^*})} + \sigma_{k^*}(2^i, \delta_{i,k^*}) \]
We get
\[ \gamma_{i-1} \leq \err(h^*, D_{i-1}) + 2\sqrt{\err(h^*, D_{i-1})
      \sigma_{k^*}(2^i, \delta_{i,k^*})} + 3\sigma_{k^*}(2^i, \delta_{i,k^*})
      \,.
      \]
  \end{enumerate}
  Thus, invariant 1 is established for iteration $i$.

Now consider the verification stage in iteration $i$.
We first prove that the loop in the verification stage will terminate and establish some properties
upon termination.
Observe that $k$ and $S$ are initially $k_{i-1}$ and $S_{i-1}$, respectively.
Throughout the loop, the examples added to $S$ are obtained from \SEARCH, and hence are consistent
with $h^*$. In addition, we have the following claim regarding $k^*$.
\begin{claim}
If invariants 1--5 holds for iteration $i-1$, then for iteration $i$, the following holds:
\begin{enumerate}[label=(\alph*),leftmargin=*]
\item $\min_{h \in H_{k^*}(S)} \err(h, T_i) \leq \gamma_{i-1} + \sqrt{\gamma_{i-1} \sigma_{k^*}(2^i, \delta_{i,k^*})} + \sigma_{k^*}(2^i, \delta_{i,k^*})$
\item \begin{multline*}
    h^* \in V_i^{k^*} = \bigl\{h \in H_{k^*}(S): \\ \err(h, T_i) \leq \min_{h' \in
    H_{k^*}(S)} \err(h', T_i) + 2\sqrt{\err(h', T_i) \sigma_{k^*}(2^i,
    \delta_{i,k^*})} + 3\sigma_{k^*}(2^i, \delta_{i,k^*}) \bigr\}
    \,.
\end{multline*}
\end{enumerate}
  \label{claim:alarch-helper}
\end{claim}
\begin{proof}
Recall that $h^*$ is the optimal hypothesis under distribution $D_{i-1}$. We have already shown above that $\err(h^*, D_{i-1}) \leq \gamma_{i-1}$.
By the definition of $E_i$,
\begin{eqnarray*}
    \min_{h \in H_{k^*}(S)} \err(h, T_i)
    &\leq&
    \err(h^*, T_i) \\
    &\leq& \err(h^*, D_{i-1})
    +
    \sqrt{\err(h^*, D_{i-1})\sigma(2^i,\delta_{i, k^*})}
    +
    \sigma(2^i,\delta_{i, k^*})\\
    &\leq&
    \gamma_{i-1}
    +
    \sqrt{\gamma_{i-1} \sigma_{k^*}(2^i, \delta_{i, k^*})}
    +
    \sigma_{k^*}(2^i, \delta_{i, k^*})
\end{eqnarray*}
where the last inequality is from that $\err(h^*, D_{i-1}) \leq \gamma_{i-1}$. This proves item (a).

On the other hand, for all $h'$ in $H_{k^*}(S)$,
\begin{eqnarray*}
    \err(h^*, T_i) &\leq& \err(h^*, D_{i-1})
    +
    \sqrt{\err(h^*, D_{i-1})\sigma_{k^*}(2^i,\delta_{i, k^*})}
    +
    \sigma(2^i,\delta_{i, k^*})\\
    &\leq& \err(h', D_{i-1})
    +
    \sqrt{\err(h', D_{i-1}) \sigma_{k^*}(2^i,\delta_{i, k^*})}
    +
    \sigma(2^i,\delta_{i, k^*})\\
    &\leq& \err(h', T_i)
    +
    2\sqrt{\err(h', T_i) \sigma_{k^*}(2^i,\delta_{i, k^*})}
    +
    3\sigma(2^i,\delta_{i, k^*})
    \,.
\end{eqnarray*}
where the first inequality is from the definition of $E_i$, the second inequality is from Invariant 5 of iteration $i-1$, the third inequality is from the definition of $E_i$.

Thus, $\err(h^*, T_i) \leq \min_{h' \in H_{k^*}(S)} \err(h', T_i) + 2\sqrt{\err(h', T_i) \sigma_{k^*}(2^i,\delta_{i, k^*})} + 3\sigma_{k^*}(2^i,\delta_{i, k^*})$, proving item (b).
\end{proof}

Claim~\ref{claim:alarch-helper} implies that $k$ cannot increase beyond $k^*$.
To see this, observe that
Claim~\ref{claim:alarch-helper}(a) implies the condition
in Step~\ref{step:check-error-alarch} is not satisfied for $k=k^*$.
In addition, Claim~\ref{claim:alarch-helper}(b) implies that $h^* \in V_i^{k^*}
\neq \emptyset$, which in turn means that $\SEARCH_{H_{k^*}}(V_i^{k^*}) = \bot$.
Hence, the loop in the verification stage would terminate if $k$ ever reaches $k^*$.
Because iteration $i$ starts with $k \leq k^*$ (as invariant 3 holds in
iteration $i-1$), invariants 2 and 3 must also hold for iteration $i$.


Finally, we can establish invariants 4 and 5 for iteration $i$.
Because the loop terminates with $\SEARCH_{H_k}(V_i^{k_i})$ returning $\bot$,
 there is no counterexample $x \in \calX$ such
that $h^*$ disagrees with every $h \in V_i^{k_i}$.
This implies that $h^*(x) = V_i^{k_i}(x)$ for all $x \notin \DIS(V_i^{k_i})$
(i.e., invariant 4).
Hence, for any hypothesis $h$,
\[ \err(h, D_i)
  \ = \
  \Pr[h(x) \neq h^*(x), x \notin \DIS(V_{i}^{k_i})] + \Pr[h(x) \neq y, x \in
  \DIS(V_{i}^{k_i})] \,.
\]
Therefore,
\begin{align*}
  \lefteqn{
    \err(h, D_i) - \err(h^*, D_i)
  } \\
  & \ = \ \Pr[h(x) \neq h^*(x), x \notin \DIS(V_i^{k_i})] + \Pr[h(x) \neq y, x
  \in \DIS(V_i^{k_i})] \\
  & \qquad - \Pr[h^*(x) \neq y, x \in \DIS(V_i^{k_i})] \\
  & \ \geq \ \Pr[h(x) \neq y, x \notin \DIS(V_i^{k_i})] - \Pr[h^*(x) \neq y, x \notin \DIS(V_i^{k_i})] \\
  & \qquad + \Pr[h(x) \neq y, x \in \DIS(V_i^{k_i})] - \Pr[h^*(x) \neq y, x \in \DIS(V_i^{k_i})] \\
  & \ = \ \err(h, D) - \err(h^*, D)
  \,,
\end{align*}
which proves invariant 5 for iteration $i$.
\end{proof}

\begin{proof}[Proof of Theorem~\ref{thm:restatealarch}]
Supose event $E$ happens.

We first show a claim regarding the error of hypotheses in current
version spaces.
\begin{claim}
On event $E$, for all $i \geq 1$, for all $h \in V_i^{k_i}$,
\[ \err(h, D) \leq \err(h^*, D_{i-1}) + 6\sqrt{\err(h^*, D_{i-1}) \sigma_{k^*}(2^i, \delta_{i,k^*})} + 21\sigma_{k^*}(2^i, \delta_{i,k^*}). \]
\label{claim:alarchvs}
\end{claim}
\begin{proof}
First, for every $h$ in $V_i^{k_i}$,
\[ \err(h, T_i) \leq \min_{h' \in H_{k_i}(S)} \err(h', T_i) + 2\sqrt{\err(h', T_i) \sigma_{k_i}(2^i,\delta_{i,k_i})} + 3\sigma_{k_i}(2^i,\delta_{i,k_i})
  \,,
  \]
and since the condition in step~\ref{step:check-error-alarch} is not satisfied for $k = k_i$, we know that
\[ \min_{h' \in H_{k_i}(S)} \err(h', T_i) \leq \gamma_{i-1} + \sqrt{\gamma_{i-1}\sigma_{k_i}(2^i,\delta_{i,k_i})} + \sigma_{k_i}(2^i,\delta_{i,k_i})
  \,. \]
Thus,
\begin{equation}
\err(h, T_i) \leq \gamma_{i-1} + 3\sqrt{\gamma_{i-1}\sigma_{k_i}(2^i,\delta_{i,k_i})} + 6\sigma_{k_i}(2^i,\delta_{i,k_i})
\,.
\label{eqn:erremph}
\end{equation}
By definition of event $E_i$, we also have
\[ \err(h, D_{i-1}) \leq \err(h, T_i) + \sqrt{\err(h, T_i)\sigma_{k_i}(2^i,\delta_{i,k_i})} + \sigma_{k_i}(2^i,\delta_{i,k_i})
  \,.
  \]
Hence,
\[ \err(h, D_{i-1}) \leq \gamma_{i-1} +
  4\sqrt{\gamma_{i-1}\sigma_{k_i}(2^i,\delta_{i,k_i})} +
  10\sigma_{k_i}(2^i,\delta_{i,k_i})
  \,.
  \]
Furthermore, by item 1 of Lemma~\ref{lem:invkagnostic},
\[
\gamma_{i-1} \leq \err(h^*, D_{i-1}) + 2\sqrt{\err(h^*, D_{i-1}) \sigma_{k^*}(2^i, \delta_{i, k^*})} + 3\sigma_{k^*}(2^i, \delta_{i, k^*})
  \,.
  \]
This implies that
\begin{eqnarray*}
  \err(h, D_{i-1}) &\leq& \err(h^*, D_{i-1}) +
  6\sqrt{\err(h^*, D_{i-1})\sigma_{k_i}(2^i,\delta_{i,k_i})} +
  21\sigma_{k_i}(2^i,\delta_{i,k_i}) \\
  &\leq& \err(h^*, D_{i-1}) +
  6\sqrt{\err(h^*, D_{i-1})\sigma_{k^*}(2^i,\delta_{i,k^*})} +
  21\sigma_{k^*}(2^i,\delta_{i,k^*})
  \,.
  \qedhere
\end{eqnarray*}
where the second inequality is from item 3 of Lemma~\ref{lem:invkagnostic}.
\end{proof}

We first prove the error rate guarantee.
Suppose iteration $i = I = \log_2 M(\nu,k^*,\epsilon,\delta)$ has been reached. Observe that from Claim~\ref{claim:alarchvs},
for $\hat{h} \in V_I^{k_I}$,
\[ \err(\hat{h}, D_{I-1}) - \err(h^*, D_{I-1}) \leq
  6\sqrt{\err(h^*, D_{I-1})\sigma_{k^*}(2^I,\delta_{I,k^*})} +
  21\sigma_{k^*}(2^I,\delta_{I,k^*}) \leq \epsilon \]
where the second inequality is from that $m = 2^I = M(\nu,k^*,\epsilon,\delta)$.
Thus, by item 5 of Lemma~\ref{lem:invkagnostic},
\[
  \err(\hat{h}, D) - \err(h^*, D) \leq \err(\hat{h}, D_{I-1}) - \err(h^*,
  D_{I-1}) \leq \epsilon \,.
\]

Next, we prove the bound on the number of \SEARCH queries.
From Lemma~\ref{lem:invkagnostic}, Algorithm~\ref{alg:budgetalarch} maintains the invariant that $k \leq k^*$.
For each iteration $i$, each call to $\SEARCH$ either returns an example forcing $k$ to increment, or returns $\bot$ which causes an exit from the verification stage loop. Therefore, the total number of $\SEARCH$ calls is at most
\[ k^* + I \leq k^* + O \del{\log\frac{d_{k^*}}{\epsilon^2} + \log\log\frac{k^*}{\delta}}. \]

Finally, we prove the bound on the number of \LABEL queries.
This is done in a few steps.
  \begin{enumerate}[leftmargin=*]
    \item We first show that the version space $V_i^{k_i}$ is always contained
    in a ball of small radius (with respect to the disagreement pseudometric
    $\Pr_{x \sim D_\calX}[h(x) \neq h'(x)]$).
      Specifically, for $1 \leq i \leq I$, for any $h, h' \in V_i^{k_i}$, from Claim~\ref{claim:alarchvs},
      in conjunction with triangle inequality,
      and that $\err(h^*, D_{i-1}) \leq \err(h, D) = \nu$,
      \begin{eqnarray*}
      && \Pr_{x \sim D_\calX}[ h(x) \neq h'(x) ] \\
      &\leq& 2\err(h^*, D_{i-1}) +
        12\sqrt{\err(h^*, D_{i-1})\sigma_{k^*}(2^i,\delta_{i,k^*})} +
        42\sigma_{k^*}(2^i,\delta_{i,k^*}) \\
      &\leq& 2\nu +
        12\sqrt{\nu \sigma_{k^*}(2^i,\delta_{i,k^*})} +
        42\sigma_{k^*}(2^i,\delta_{i,k^*})
    \end{eqnarray*}

%
%
Thus, $V_i^{k_i}$ is contained in $\B_{H_{k_i}(S)}(h, 2\nu + 12\sqrt{\nu \sigma_{k^*}(2^i,\delta_{i, k^*})} + 42\sigma_{k^*}(2^i,\delta_{i, k^*}))$ for some $h$ in $H_{k_i}(S)$.

\item Next we bound the label complexity per iteration. Note that by the choice of $m = 2^I = M(\nu,k^*,\epsilon,\delta)$,
$6\sqrt{\nu \sigma_{k^*}(2^{I-1},\delta_{I-1, k^*})} + 21\sigma_{k^*}(2^{I-1},\delta_{I-1, k^*}) \geq \epsilon$, therefore
for all $1 \leq i \leq I$, $6\sqrt{\nu \sigma_{k^*}(2^i,\delta_{i, k^*})} + 21\sigma_{k^*}(2^i,\delta_{i, k^*}) \geq \epsilon/2$.
Thus, the size of the disagreement region can be bounded as
\begin{eqnarray}
\Pr_{x \sim D_\calX} [x \in \DIS(V_i^{k_i})] &\leq& \theta_k(2\nu + \epsilon)
  \cdot \del{ 2\nu + 12\sqrt{\nu \sigma_{k^*}(2^i, \delta_{i, k^*})} + 42\sigma_{k^*}(2^i,
  \delta_{i, k^*}) } \notag \\
&\leq& \theta_k(2\nu + \epsilon)
  \cdot \del{ 8\nu + 48\sigma_{k^*}(2^i, \delta_{i, k^*}) }
\,.
  \label{eq:alarch-label-query-prob}
\end{eqnarray}
By definition of $E_i$, the number of queries to $\LABEL$ at iteration $i$ is at most
\[ 2^{i+1} \Pr_{x \sim D_\calX}[x \in \DIS(V_i^{k_i})] + O\del{\sqrt{2^{i+1}
\Pr_{x \sim D_\calX}[x \in \DIS(V_i^{k_i})] \log(1/\delta_{i, k^*})} +
\log(1/\delta_{i, k^*})} \,. \]
Combining this with~\eqref{eq:alarch-label-query-prob} gives
\begin{equation}
  \text{\# \LABEL queries in iteration $i$}
  \ = \
  O\del{2^i \cdot \theta_{k_i}(2\nu + \epsilon) \cdot (\nu + \sigma_{k^*}(2^i, \delta_{i, k^*})) } \,.
  \label{eq:alarch-label-queries}
\end{equation}

  \item From the setting of $m = 2^I = \tilde{O}(d_{k^*} (\nu + \epsilon)/\epsilon^2)$, we get that
\[ I = O \del{\log\frac{d_{k^*}}{\epsilon} + \log\log\frac{k^*}{\delta}} \,. \]
Now, using~\eqref{eq:alarch-label-queries}, we get that the total number of $\LABEL$ queries by Algorithm~\ref{alg:budgetalarch} is bounded by
\begin{align*}
\lefteqn{
    2 + \sum_{i=1}^I O\del{2^i \cdot \theta_{k_i}(2\nu + \epsilon) \cdot (\nu + \sigma_{k^*}(2^i, \delta_{i, k^*}))}
    }\\
    &= 2 + \sum_{i=1}^I O\del{2^i \cdot \max_{k \leq k^*} \theta_k(2\nu + \epsilon) \cdot (\nu + \sigma_{k^*}(2^i, \delta_{i, k^*})) } \\
    &= O\del{\max_{k \leq k^*}\theta_k(2\nu + \epsilon) \cdot \left(\sum_{i=1}^I 2^i (\nu + \sigma_{k^*}(2^i, \delta_{i, k^*}))\right) } \\
    &= O\del{\max_{k \leq k^*}\theta_k(2\nu + \epsilon) \cdot \left(\nu 2^I + \sum_{i=1}^I 2^i \frac{d\ln(2^i) + \ln(\frac{(i^2+i)(k^*)^2}{\delta})}{2^i} \right) } \\
    &= O\del{\max_{k \leq k^*}\theta_k(2\nu + \epsilon) \cdot \left(\nu 2^I + d_{k^*} I^2 + I \log\frac{k^*}{\delta} \right) }\\
    &= O
    \left(
      \max_{k \leq k^*}\theta_k(2\nu + \epsilon) \cdot
      \left(
        \frac{\nu^2 +
        \epsilon\nu}{\epsilon^2}
        \del{ d_{k^*}\log\frac{1}{\epsilon} + \log\frac{k^*}{\delta} }
        + d_{k^*}
        \del{\log\frac{d_{k^*}}{\epsilon} + \log\log\frac{k^*}{\delta}}^2
      \right.
    \right.
    \\
    & \qquad\qquad\qquad\qquad\qquad\qquad\qquad
    \left.
      \left. +
      \del{\log\frac{d_{k^*}}{\epsilon} + \log\log\frac{k^*}{\delta}}
      \log\frac{k^*}{\delta} \right)
    \right)
    \\
    &= \tilde{O}\del{\max_{k \leq k^*}\theta_k(2\nu + \epsilon) \cdot \del{d_{k^*}(\log\frac{1}{\epsilon})^2 + \log\frac{k^*}{\delta}} \cdot \del{1 + \frac{\nu^2}{\epsilon^2}}}
    .
    \qedhere
\end{align*}
\end{enumerate}
\end{proof}

\section{Performance Guarantees of \algaa}
\label{appendix-amortized}

\subsection{Detailed Description of Subroutines}

Subroutine \SAL performs standard disagreement-based selective sampling.
Specifically, it draws an unlabeled example $x$ from the $D_{\calX}$. If
$x$ is in the agreement region of version space $V$, its label is inferred
as $V(x)$; otherwise, we query the $\LABEL$ oracle to get its label. The counter
 $c$ is incremented when $\LABEL$ is called.

\begin{algorithm}[H]
 \caption{\SAL}
 \label{alg:sal}
  \begin{algorithmic}[1]
    \REQUIRE Version space $V \subset H$, oracle $\LABEL$, labeled dataset $L$, counter $c$.
    \ENSURE New labeled dataset $L'$, new counter $c'$.

    \STATE $x \gets \text{independent draw from $D_\calX$}$
    (the corresponding label is hidden).

    \IF{$x \in \DIS(V)$}

      \STATE $L' \gets L \cup \cbr{(x,\LABEL(x))}$

      \STATE $c' \gets c + 1$

    \ELSE

      \STATE $L' \gets L \cup \cbr{(x,V(x))}$

      \STATE $c' \gets c$

    \ENDIF

  \end{algorithmic}
\end{algorithm}
Subroutine \EC checks if the version space has high error, based on item 2
of Lemma~\ref{lem:hsgood} -- that is, if $k = k^*$, then \EC should never fail.
Furthermore, if version space $V_i$ fails \EC, then $V_i$ should have small
radius -- see Lemma~\ref{lemma:errh} for details.
\begin{algorithm}[H]
  \caption{\EC}
  \label{alg:ec}
  \begin{algorithmic}[1]
  \REQUIRE Version space $V \subset H_k$, labeled dataset $L$ of size $l$, confidence $\delta$.
  \ENSURE Boolean variable $b$ indicating if $V$ has high error.

  \STATE Let $\delta_k := \delta/((k+1)(k+2))$ for all $k \geq 0$.

  \STATE $\gamma \gets \min_{k' \geq k, h \in H_{k'}}\cbr[1]{\err(h, L) + 2\sqrt{\err(h, L)\sigma_{k'}(l, \delta_{k'})} + 3\sigma_{k'}(l, \delta_{k'})}$
  \label{alg:ec:gamma}


  \IF{$\min_{h \in V} \err(h, L) > \gamma + 2\sqrt{\gamma \sigma_k(l, \delta_{k})} + 3\sigma_k(l, \delta_{k})$}
  \label{alg:ec:vserr}

    \STATE $b \gets \TRUE$

  \ELSE

    \STATE $b \gets \FALSE$

  \ENDIF

  \end{algorithmic}
\end{algorithm}

Subroutine \PVS performs update on our version space based on standard generalization
error bounds.  The version space never eliminates the optimal hypothesis in $H_k(S)$
when working with $H_k$. Claim~\ref{claim:kstar} shows that, if at step $i$, $k = k^*$, then
$h^* \in V_i$ from then on.

\begin{algorithm}[H]
  \caption{\PVS}
  \label{alg:pvs}
  \begin{algorithmic}[1]
  \REQUIRE Version space $V \subset H_k$, labeled dataset $L$ of size $l$, confidence $\delta$.
  \ENSURE Pruned version space $V'$.

  \STATE Update version space:
    \[
      \kern-35pt
      V' \gets \cbr[2]{h \in V: \err(h, L) \leq \min_{h' \in V} \err(h', L) + 2\sqrt{\err(h', L)\sigma_{k}(l, \delta_k)} + 3\sigma_{k}(l, \delta_k) },
    \]
    \label{alg:pvs:prune}
    where $\delta_k := \frac{\delta}{(k+1)(k+2)}$.

  \end{algorithmic}
\end{algorithm}

Subroutine \UVS is called when (1) a systematic mistake of the version space $V_i$
has been found by \SEARCH; or (2) \EC detects that the error of $V_i$ is high. In
either case, $k$ can be increased to the minimum level such that the updated
$H_k(S)$ is nonempty. This still maintains the invariant that $k \leq k^*$.

\begin{algorithm}[H]
 \caption{\UVS}
 \label{alg:uvs}
  \begin{algorithmic}[1]
    \REQUIRE Current level of hypothesis class $k$, seed set $S$, seed to be added $s$.
    \ENSURE New level of hypothesis class $k$, new seed set $S$, updated version space $V$.

      \STATE $S \gets S \cup s$

      \STATE $k \gets \min\cbr{k' > k: H_{k'}(S) \neq \emptyset}$

      \STATE $V \gets H_k(S)$

  \end{algorithmic}
\end{algorithm}

\subsection{Proof of Theorem~\ref{thm:amortized}}

This section uses the following definition of $\sigma$:
\[
\sigma_k(m,\delta)=\phi(d_k,m,\delta/3)=\frac{1}{m}(d\log em^2 + \log \frac{6}{\delta}).
\]

We restate Theorem~\ref{thm:amortized} here for convenience.
\begin{theorem}
There exist constants $c_1, c_2 > 0$ such that the following holds.
Assume $\err(h^*) = \nu$. Let $\theta_{k'}(\cdot)$ denote the disagreement coefficient of $V_i$ at
the first step $i$ after which $k \geq k'$.
Fix any $\epsilon,\delta \in (0,1)$. Let
$n_\epsilon = c_1 \max_{k \leq k^*}\theta_k(2\nu + 2\epsilon) (d_{k^*}\log\frac{1}{\epsilon} + \log\frac{1}{\delta}) (1 + \nu^2/\epsilon^2)$
and define $C_\epsilon = 2(n_\epsilon +k^*\tau)$.
Run Algorithm~\ref{alg:aalarch} with a nested sequence of hypotheses $\cbr{H_k}_{k=0}^{\infty}$,
oracles $\LABEL$ and $\SEARCH$,
confidence parameter $\delta$,
cost ratio $\tau \geq 1$,
and upper bound
$N = c_2(d_{k^*}\log\frac{1}{\epsilon} + \log\frac{1}{\delta})/\epsilon^2$.
If the cost spent is at least $C_\epsilon$,
then with probability $1-\delta$, the current hypothesis $\tilde{h}$ has error at most $\nu + \epsilon$.
\label{thm:restateamortized}
\end{theorem}

\paragraph{Remark.}
The purpose of having a bound on unlabeled examples, $N$, is rather technical---
to deter the algorithm from getting into an infinite loop due to its blind self-confidence. Suppose that \algaa
starts with $H_0$ that has a single element $h$. Then, without such an $N$-based
condition, 
it will incorrectly infer the labels of all the unlabeled examples drawn
and end up with an infinite loop between lines~\ref{step:aalarch-start-repeat}
and~\ref{step:aalarch-end-repeat}.
The condition on $N$ is very mild---any $N$ satisfying $N = \poly(d_{k^*}, 1/\epsilon)$ and $N= \Omega(d_{k^*}/\epsilon^2)$ is sufficient.

\hide{
Although the $N$ appears to be an unnatural parameter
 in Algorithm~\ref{alg:aalarch},
 its requirement is quite mild. In fact, as can be seen from the proof,
any $N$ satisfying $N = \poly(d_{k^*}, 1/\epsilon)$
and $N= \Omega(d_{k^*}/\epsilon^2)$
is sufficient
to show the desired label complexity bound up to constant factors.
See Lemma~\ref{lem:k} for details.
The purpose of having a bound on unlabeled examples $N$ is rather technical:
to avoid ``blind self-confidence'' of the algorithm. For example, suppose \algaa
starts with $H_0$ that has a single element $h$. Then, without the $N$-based
condition to exit the repeat loop, 
it will incorrectly infer labels of all the unlabeled examples drawn
and end up with an infinite loop between lines~\ref{step:aalarch-start-repeat}
and~\ref{step:aalarch-end-repeat}.
}


\begin{proof}[Proof of Theorem~\ref{thm:restateamortized}]
For integer $j \geq 0$, define step $j$ as the execution period in \algaa
when the value of $i$ is $j$.

Let $l_i=|L_i|$.
Denote by $L_i^D$ the dataset containing unlabeled examples in $L_i$
labeled entirely by $\LABEL$, i.e.,
$ L_i^D = \cbr{(x, \LABEL(x)) : (x,y) \in L_i} $.
Note that $L_i^D$ is an iid sample from $D$.

We call dataset $L_i$ has {\em favorable bias}, if the following holds for any hypothesis $h$:
\begin{equation}
\err(h, L_i^D) - \err(h^*, L_i^D) \leq \err(h, L_i) - \err(h^*, L_i).
\label{eqn:favbias}
\end{equation}

Let $E_i$ be the event that the following conditions hold:
\begin{enumerate}
  \item
    For every $k \geq 0$, every $h \in H_k$ satisfies
    \[
      \err(h, D)
      \ \leq \
      \err(h, L_i^D)
      +
      \sqrt{ \err(h, L_i^D)
      \sigma_k(l_i,\delta_{i,k})}
      +
      \sigma_k(l_i,\delta_{i,k})
      \,,
    \]
    \[
    \err(h, L_i^D)
    \ \leq \
    \err(h, D)
    +
    \sqrt{ \err(h, D)
    \sigma_k(l_i,\delta_{i,k})}
    +
    \sigma_k(l_i,\delta_{i,k})
    \,.
    \]
    For every $h, h' \in H_k$,
    \begin{align*}
    (\err(h, L_i^D) &- \err(h', L_i^D))
    -
    (\err(h, D) - \err(h', D))\\
    \ \leq & \
    \sqrt{
	d_{L_i^D}(h,h') \cdot\sigma_k(l_i,\delta_{i,k})}
    +
    \sigma_k(l_i,\delta_{i,k})
    \,.
    \end{align*}
   where $d_{L_i^D}(h,h') = \frac{1}{l_i}\sum_{(x,y)\in L_i^D}[h(x)\neq h'(x)]$, fraction of $L_i^D$ where $h$ and $h'$ disagree.
  \item
  For every $1 \leq i' < i$, the number of $\LABEL$ queries from step $i'$ to step $i$ is at most
    \[
       \sum_{j=i'}^i \Pr_{x \sim D_\calX}[x \in \DIS(V_{j-1})]+ O\del{
        \sqrt{ \sum_{j=i'}^i \Pr_{x \sim D_\calX}[x \in \DIS(V_{j-1})] \log(1/\delta_i)} + \log(1/\delta_i)
      }
      \,,
    \]
where $V_j$ denotes its final value in Algorithm~\ref{alg:aalarch}.

\end{enumerate}
Using Theorem~\ref{thm:vc} and Lemma~\ref{lem:invfreed}, along with the union bound, $\Pr(E_i) \geq 1 - \delta_i$.
Define $E: =\cap_{i=1}^{\infty} E_i$, by union bound, $\Pr(E) \geq 1 - \delta$.
We henceforth condition on $E$ holding.

Define
\begin{eqnarray*}
M(\nu, k^*,\epsilon,\delta, N) &:=& \min\cbr{ m \in \mathbb{N} : 8\sqrt{\nu \sigma_{k^*}(m, \delta_{m+k^*N,{k^*}})} + 35\sigma_{k^*}(m, \delta_{m+k^*N,{k^*}}) \leq \epsilon } \\
&\leq&
O\del{
    \frac{
      (d_{k^*}\log(1/\epsilon) + \log(N k^*/\delta))
      (\nu + \epsilon)
    }{\epsilon^2}
  }
\end{eqnarray*}

We say that an iteration of the loop is \emph{verified} if Step~\ref{step:aa-verify} is triggered;
all other iterations are \emph{unverified}.
Let $\Gamma$ be the set of $i$'s where $x_i$ gets added to the final set $L$, and $\Delta$ be
the set of $i$'s where $x_i$ gets discarded. It is easy to see that if $i$ is in $\Gamma$ (resp. $\Delta$),
then the $i$ is in a verified (resp. unverified) iteration.

Define
$ i^* := \min\cbr{i \in \Gamma: l_i \geq M(\nu, k^*, \epsilon, \delta, N)}.
$
Denote by $k_i$ the final value of $k$ after $i$ unlabeled examples are processed.



We need to prove two claims:
\begin{enumerate}
\item For $i\geq i^*$, $\err(\tilde{h}_i) \leq \nu + \epsilon$, where $\tilde{h}_i$ is the hypothesis $\tilde{h}$ stored at the end of step $i$.
\item The total cost spent by Algorithm~\ref{alg:aalarch} up to step $i^*$ is at most $C_\epsilon$.
\end{enumerate}

To prove the first claim,
fix any $i \geq i^*$. The stored hypothesis $\tilde{h}_i$
is updated only when $i\in \Gamma$, so it suffices to consider only $i\in\Gamma$.
From Lemma~\ref{lem:k}, $i\leq l_i+k^*N$.
We also have $l_i \geq M(\nu, k^*, \epsilon, \delta, N)$.  Since $\tilde{h}_i \in V_i$, Lemma~\ref{lemma:errh} gives
\begin{eqnarray*}
\err(\tilde{h}_i) &\leq& \nu + 8 \sqrt{\nu \sigma_{k^*}(l_i,\delta_{i,k^*})} + 35 \sigma_{k^*}(l_i,\delta_{i,k^*}) \\
&\leq& \nu + 8 \sqrt{\nu \sigma_{k^*}(l_i,\delta_{l_i+k^*N,k^*})} + 35 \sigma_{k^*}(l_i,\delta_{l_i+k^*N,k^*}) \\
&\leq& \nu + \epsilon,
\end{eqnarray*}
as desired.


For the second claim, we first show that for $i$ in $\Gamma$, the version space is contained in a ball of small radius
(with respect to the disagreement pseudometric), thus bounding the size of its disagreement region.
Lemma~\ref{lemma:errh} shows that for $i \in \Gamma$, every hypothesis $h\in V_i$
has error at most $\nu + 8\sqrt{\nu \sigma_{k^*}(l_i, \delta_{i,k^*})} + 35\sigma_{k^*}(l_i, \delta_{i,k^*})$.

Thus, by the triangle inequality and Lemma~\ref{lem:k},
\begin{eqnarray*}
V_i &\subseteq& \B_{H_{k_i}}(h, 2\nu + 16\sqrt{\nu \sigma_{k^*}(l_i, \delta_{i,k^*})} + 70\sigma_{k^*}(l_i, \delta_{i,k^*})) \\
&\subseteq& \B_{H_{k_i}}(h, 2\nu + 16\sqrt{\nu \sigma_{k^*}(l_i, \delta_{l_i+k^*N,k^*})} + 70\sigma_{k^*}(l_i, \delta_{l_i+k^*N,k^*})).
\end{eqnarray*}
for some
$h$ in $H_{k_i}(S)$. This shows that for $i \in \Gamma$, $i \leq i^*$,
\begin{eqnarray}
&\Pr_{x \sim D_{\calX}}[x \in \DIS(V_i)] \nonumber \\
\leq& \theta_{k_i}(2\nu + 2\epsilon) \cdot \del{2\nu + 16\sqrt{\nu \sigma_{k^*}(l_i, \delta_{l_i + k^*N,k^*})} + 70\sigma_{k^*}(l_i, \delta_{l_i + k^*N,k^*})} \nonumber \\
\leq& \max_{k \leq k^*} \theta_k(2\nu + 2\epsilon) \cdot \del{2\nu + 16\sqrt{\nu \sigma_{k^*}(l_i, \delta_{l_i + k^*N,k^*})} + 70\sigma_{k^*}(l_i, \delta_{l_i + k^*N,k^*})},
\label{eqn:disbound}
\end{eqnarray}
where the first inequality is from the definition of $\theta_{k_i}(\cdot)$ and
$8\sqrt{\nu \sigma_{k^*}(l_i, \delta_{l_i+k^*N,k^*})} + 35\sigma_{k^*}(l_i, \delta_{l_i+k^*N,k^*}) \geq \epsilon$
for $i \leq i^*$, the second inequality is from $k_i \leq k^*$.

\item
For $i \geq 1$, let $Z_i$ be the indicator of whether \LABEL is queried with
$x_i$ in Step~\ref{step:aalarch-label-query}, i.e.,
\[ Z_i = \ind{x_i \in \DIS(V_{i-1})} \,. \]
For $0 \leq k \leq k_{i^*}$, define
\begin{align*}
  i_k^0 & \ := \ \min\cbr{i \leq i^*: k_i \geq k}, \,\,\textrm{the first step when the hypothesis class reaches $\geq k$}, \\
  i_k & \ := \ \max\cbr{i \leq i^*: k_i \leq k}, \,\,\textrm{the last step when the hypothesis class is still $\leq k$ by the end of that step}, \\
  i_k' & \ := \ \max\cbr{i_k^0 \leq i \leq i_k: k_i \leq k, i \in \Gamma}, \,\,\textrm{the last verified step for hypothesis class $\leq k$ (if exists)}. \\
\end{align*}
We call class $k$ {\em skipped} if there is no step $i$ such that $k_i = k$.
If level $k$ is skipped, then $i_k = i_{k-1} = i_k^0 - 1$, and $i_k'$ is undefined.

Let
\[ W_k := \sum_{i=i_k^0+1}^{i_k'} Z_i \]
be the number of verified queried examples when working with hypothesis class $H_k$. Note that $W_k / \tau$ is the number of verified iterations when working with $H_k$.
If level $k$ is skipped, then $W_k := 0$.

Let
\[ Y_k := \sum_{i=i_k'+1}^{i_k+1} Z_i \]
be the number of unverified queried examples when working with hypothesis class $H_k$. Note that $Y_k \leq \tau$, and there is at most one unverified iteration when working with $H_k$.
If level $k$ is skipped,then $Y_k := 0$.

Therefore, the total cost when working with $H_k$ is at most
\[ \frac{W_k}{\tau} \cdot 2\tau + Y_k + \tau \leq 2\tau + 2 W_k \]
Furthermore, Claim~\ref{claim:kstar} implies that
there is no unverified iteration when working with $H_{k^*}$. Hence the total cost
when working with $H_{k^*}$ has a tighter upper bound, that is, $2 W_{k^*}$.

As a shorthand, let $m = M(\nu, k^*, \epsilon, \delta, N)$. We now bound the total cost incurred up to time $i^*$ as
\begin{eqnarray*}
 \sum_{k=0}^{k^*-1}(2\tau + 2 W_k) + 2 W_{k^*}
&=& 2\tau k^* + 2\sum_{k=0}^{k_{i^*}} W_k \\
&=& 2\tau k^* + 2\sum_{k=0}^{k_{i^*}} \sum_{i=i_k^0+1}^{i_k'} Z_i \\
&=& 2\tau k^* + O\del{2\sum_{i \in \Gamma: i \leq i^*} \Pr_{x \sim D_\calX}[x \in \DIS(V_{i-1})]} + O\del{k^*\ln\frac{1}{\delta_{i^*}}} \\
&\leq& 2 \del{\tau k^* + O\del{\sum_{l = 1}^{m-1} \max_{k \leq k^*} \theta_k(2\nu + 2\epsilon) (\nu + \sigma_{k^*}(l, \delta_{l+k^*N,k^*}))} } \\
&\leq& 2 \del{\tau k^* + O\del{\max_{k \leq k^*} \theta_k(2\nu + 2\epsilon) \sum_{l=1}^{m-1} (\nu + \sigma_{k^*}(l, \delta_{l+k^*N,k^*})))} } \\
&\leq& 2 \del{\tau k^* + \tilde{O}\del{\max_{k \leq k^*}\theta_k(2\nu + 2\epsilon) d_{k^*} \left(1 + \frac{\nu^2}{\epsilon^2}\right)} }\\
&\leq& 2 \del{\tau k^* + n_\epsilon } = C_\epsilon,
\end{eqnarray*}
where the first equality is by algebra,
the second equality is from the definition of $W_k$, and
the third equality is from the definition of $E$.
The first inequalty is from Lemma~\ref{lemma:errh},
using Equation~\eqref{eqn:disbound} to bound $\Pr_{x \sim D_\calX}[x \in \DIS(V_{i-1})]$ and noting that $\cbr{l_i: i \in \Gamma, i \leq i^*}=[m]$.
\end{proof}

Now we provide the proof of our two key lemmas(Lemmas~\ref{lemma:errh} and~\ref{lem:k}).

Consider the last call of $\PVS$ in step $i$.
Define $\gamma_i$ as the value of $\gamma$ in line~\ref{alg:ec:gamma} of $\EC$:
\begin{equation}
\gamma_i = \min_{k' \geq k_i, h \in H_{k'}} \cbr{\err(h,L_i) + 2\sqrt{\err(h,L_i) \sigma_{k'}(l, \delta_{i,k'})} + 3\sigma_{k'}(l, \delta_{i,k'})}
\label{eqn:defgammai}
\end{equation}
Meanwhile, from line~\ref{alg:pvs:prune} of $\PVS$, we have for all $h \in V_i$,
\[ \err(h, L_i) \leq \min_{h' \in V_i} \err(h', L_i) + 2\sqrt{\err(h', L_i)\sigma_{k_i}(l_i, \delta_{i,k_i})} + 3\sigma_{k_i}(l_i, \delta_{i,k_i}) \]
where $V_i$ denotes its final value.

\begin{lemma}
Assume that the following conditions hold:
\begin{enumerate}
\item The dataset $L_i$ has favorable bias, i.e. it satisfies Equation~\eqref{eqn:favbias}.
\item The version space $V_i$ is such that $\EC(V_i, L_i, \delta_i)$ returns false, i.e.
it has a low empirical error on $L_i$:
\begin{equation}
\min_{h' \in V_i} \err(h', L_i) \leq \gamma_i + 2\sqrt{\gamma_i \sigma_{k_i}(l_i,\delta_{i,k_i})} + 3\sigma_{k_i}(l_i,\delta_{i,k_i}).
\label{eqn:lowerror}
\end{equation}
\end{enumerate}
Then, every $h \in V_i$ is such that
\begin{equation}
   \err(h) \leq \nu + 8 \sqrt{\nu \sigma_{k^*}(l_i,\delta_{i,k^*})} + 35 \sigma_{k^*}(l_i,\delta_{i,k^*}).
\label{eqn:errh}
\end{equation}
where $V_i$ and $L_i$ denote their final values,
respectively.
Specifically, Equation~\eqref{eqn:errh} holds for any $h \in V_i$ such that $i \in \Gamma$
or $i+1 \in \Gamma$.
\label{lemma:errh}
\end{lemma}
\begin{proof}
Lemma~\ref{lem:k} shows that $k_i \leq k^*$, which we will use below.

Start with Equation~\eqref{eqn:lowerror}:
\[
\min_{h' \in V_i} \err(h', L_i) \leq \gamma_i + 2\sqrt{\gamma_i \sigma_{k_i}(l_i,\delta_{i,k_i})} + 3\sigma_{k_i}(l_i,\delta_{i,k_i}).
\]
Since $k_i \leq k^*$, $\sigma_{k_i}(l_i,\delta_{i,k_i}) \leq \sigma_{k^*}(l_i,\delta_{i,k^*})$.

From the definition of $\gamma_i$ (Equation~\eqref{eqn:defgammai}), taking $k = k^* \geq k_i$, $h = h^* \in H_{k^*}$,
\[
\gamma_i \leq \err(h^*, L_i) + 2\sqrt{\err(h^*, L_i)\sigma_{k^*}(l_i,\delta_{i,k^*})} + 3\sigma_{k^*}(l_i,\delta_{i,k^*}).
\]
Plugging the latter into the former and using $\sigma$  as a shorthand for $\sigma_{k^*}(l_i,\delta_{i,k^*})$, we have
\begin{align*}
\min_{h' \in V_i}  \err(h', L_i)
 \leq & \err(h^*, L_i) + 2\sqrt{\err(h^*, L_i)\sigma} + 3\sigma
 + 2\sqrt{(\err(h^*, L_i) + 2\sqrt{\err(h^*, L_i)\sigma} + 3\sigma) \sigma} + 3\sigma \\
\leq & \err(h^*, L_i) + 2\sqrt{\err(h^*, L_i)\sigma} + 6\sigma
 + 2(\sqrt{\err(h^*, L_i)\sigma} + \sqrt{3}\sigma) \\
\leq & \err(h^*, L_i) + 4\sqrt{\err(h^*, L_i)\sigma} + 10\sigma
\,.
\end{align*}

Fix any $h \in V_i$. By construction,
\[
\err(h, L_i) \leq \min_{h' \in V_i} \err(h', L_i) + 2\sqrt{\err(h',L_i) \sigma_{k_i}(l_i,\delta_{i,k_i})} + 3\sigma_{k_i}(l_i,\delta_{i,k_i}).
\]

Plugging the former into the latter (recalling that $\sigma_{k_i}(l_i,\delta_{i,k_i})\leq \sigma$) gives
\begin{align*}
\err(h, L_i) - \err(h^*, L_i) \leq & 4\sqrt{\err(h^*, L_i)\sigma} + 10\sigma
 + 2(\sqrt{\err(h^*,L_i) \sigma} + \sqrt{10}\sigma) + 3\sigma \\
 \leq &  6\sqrt{\err(h^*, L_i)\sigma} + 20\sigma.
\end{align*}
Combined with Equation~\eqref{eqn:favbias}, we have,
\[ \err(h, L_i^D) - \err(h^*, L_i^D) \leq 6 \sqrt{\err(h^*, L_i)
  \sigma} + 20 \sigma. \]
Since $\err(h^*, L_i) \leq \err(h^*, L_i^D)$,
\[ \err(h, L_i^D) - \err(h^*, L_i^D) \leq 6 \sqrt{\err(h^*, L_i^D)
  \sigma} + 20 \sigma. \]
From the definition of $E_i$,
\[ \err(h^*, L_i^D) \leq \nu + \sqrt{\nu \sigma} + \sigma. \]
\[ \err(h) \leq \err(h, L_i^D) + \sqrt{ \err(h,
  L_i^D)\sigma} + \sigma. \]
Plugging in and simplifying algebratically gives
\[ \err(h) \leq \nu + 8 \sqrt{\nu \sigma} + 35 \sigma.\]
Now, if $i \in \Gamma$, the dataset $L_i$ has favorable bias
from lemma~\ref{lem:favbias}; if $i \notin \Gamma$ and $i+1 \in \Gamma$,
the final value of $L_i$ equals some $L_j$ for some $j \in \Gamma$, therefore
also has favorable bias.

Meanwhile, if $i \in \Gamma$, Algorithm~\ref{alg:aalarch} fails
$\EC(V_i, L_i, \delta_i)$
 for $k = k_i$. If $i \notin \Gamma$ and $i+1 \in \Gamma$, then $i+1$ is
the start of some verified iteration, i.e. $i+1 = i_k^0$ for some $k$.
Hence the final value of $V_i$ also fails $\EC(V_i, L_i, \delta_i)$
for $k = k_i$. In both cases, Equation~\eqref{eqn:lowerror} holds.

Therefore, if $i \in \Gamma$ or $i+1 \in \Gamma$, then
Equation~\eqref{eqn:errh} holds for every $h$ in $V_i$.
\end{proof}

\begin{lemma}
For step $i$, suppose $L_i$ has favorable bias, i.e. Equation~\eqref{eqn:favbias} holds.
Then for any $k$ and any $h\in H_k$,
\[
\err(h^*, L_i) - \err(h, L_i) \leq 2\sqrt{\err(h, L_i) \sigma_{\bar{k}}(l_i,\delta_{i,\bar{k}})} + 3\sigma_{\bar{k}}(l_i,\delta_{i,\bar{k}}),
\]
where $\bar{k} = \max(k^*, k)$.
Specifically:
\begin{enumerate}
  \item for any $h \in H_{k^*}$,
  \begin{equation}
  \err(h^*, L_i) - \err(h, L_i) \leq 2\sqrt{\err(h, L_i) \sigma_{k^*}(l_i,\delta_{i,k^*})} + 3\sigma_{k^*}(l_i,\delta_{i,k^*}),
  \label{eqn:hsinvks}
  \end{equation}

  \item The empirical error of $h^*$ on $L_i$ can be bounded as follows:
  \begin{equation}
     \err(h^*, L_i) \leq \gamma_i + 2\sqrt{\gamma_i \sigma_{k^*}(l_i, \delta_{i,k^*})} + 3\sigma_{k^*}(l_i, \delta_{i,k^*})
     \label{eqn:noline11}
  \end{equation}

\end{enumerate}

\label{lem:hsgood}
\end{lemma}
\begin{proof}
Fix any $k$ and $h\in H_k$.
Since $\bar{k}\geq k$,
$\sigma_k(l_i,\delta_{i,k}) \leq \sigma_{\bar{k}}(l_i,\delta_{i,\bar{k}})$.  Similarly, $\sigma_{k^*}(l_i,\delta_{i,k^*}) \leq \sigma_{\bar{k}}(l_i,\delta_{i,\bar{k}})$.
Using the shorthand $\sigma := \sigma_{\bar{k}}(l_i,\delta_{i,\bar{k}})$ and noting that $h, h^*\in H_{\bar{k}}$,
\begin{eqnarray*}
\err(h^*, L_i) - \err(h, L_i) &\leq& \err(h^*, L_i^D) - \err(h, L_i^D) \\
&\leq& \sqrt{d_{L_i^D}(h^*,h) \cdot \sigma} + \sigma \\
&\leq& \sqrt{(\err(h^*, L_i) + \err(h, L_i)) \cdot \sigma} + \sigma. \\
&\leq& \sqrt{\err(h^*, L_i)\sigma} + \sqrt{\err(h, L_i)\sigma} + \sigma.
\end{eqnarray*}
where the first inequality is from Equation~\eqref{eqn:favbias}, the second inequality is from the definition of $E_i$ and the optimality of $h^*$, and the third inequality is from the triangle inequality.
Letting $A=\err(h^*, L_i)$, $B=\err(h, L_i)$, and $C=B+\sqrt{B\sigma} +\sigma$, we can rewrite the above inequality as $A\leq C+\sqrt{A\sigma}$.  Solving the resulting quadratic equation in terms of $A$, we have $A \leq C + \sigma + \sqrt{C\sigma}$, or
\begin{eqnarray*}
A & \leq & B+\sqrt{B\sigma} + 2\sigma + \sqrt{\sigma(B+\sqrt{B\sigma} +\sigma)} \\
  & \leq & B+\sqrt{B\sigma} + 2\sigma + \sqrt{\sigma}(\sqrt{B}+\sqrt{\sigma}) \\
  & \leq & B+ 2\sqrt{B\sigma} + 3\sigma,
\end{eqnarray*}
or
\[
\err(h^*, L_i) \leq \err(h, L_i) + 2\sqrt{\err(h, L_i)\sigma} + 3\sigma.
\]

Specifically:
\begin{enumerate}
\item Taking $k = k^*$, we get that Equation~\eqref{eqn:hsinvks} holds
for any $h \in H_{k^*}$, establishing item 1.

\item Define
\[ (\hat{k}_i, \hat{h}_i) := \arg\min_{k' \geq k^*, h \in H_{k'}}\cbr{\err(h, L_i) + 2\sqrt{\err(h, L_i)\sigma_{k'}(l_i, \delta_{i,k'})} + 3\sigma_{k'}(l_i, \delta_{i,k'})}. \]
In this notation,
$\gamma_i = \err(\hat{h}_i, L_i) + 2\sqrt{\err(\hat{h}_i, L_i) \sigma_{\hat{k}_i}(l_i,\delta_{i,\hat{k}_i})} + 3\sigma_{\hat{k}_i}(l_i,\delta_{i,\hat{k}_i})$.
We have
\begin{align*}
\gamma_i + 2\sqrt{\gamma_i \sigma_{k^*}(l_i, \delta_{i,k^*})} + 3\sigma_{k^*}(l_i, \delta_{i,k^*}) &\geq \err(\hat{h}_i, L_i) + 2\sqrt{\err(\hat{h}_i, L_i) \sigma_{\bar{k}}(l_i, \delta_{i,\bar{k}})} + 3\sigma_{\bar{k}}(l_i, \delta_{i,\bar{k}}) \\
&\geq \err(h^*, L_i),
\end{align*}
where $\bar{k} = \max(k^*, \hat{k}_i)$ and
the last inequality comes from applying Lemma~\ref{lem:hsgood} for $h' = \hat{h}_i \in H_{\hat{k}_i}$ and $\bar{k}$.
This establishes Equation~\eqref{eqn:noline11}, proving item 2. \qedhere
\end{enumerate}
\end{proof}


\begin{lemma}
At any step of \algaa, $k\leq k^*$.
Consequently, for every $i$, $i \leq l_i + k^* N$.
\label{lem:k}
\end{lemma}
\begin{proof}
We prove the lemma in two steps.
\begin{enumerate}[leftmargin=*]
\item Notice that there are two places where $k$ is incremented in \algaa, line~\ref{step:aalarch-err-uvs}
and line~\ref{step:aalarch-seed-uvs}.
If $k<k^*$, neither line would increment it beyond $k^*$ as $h^*\in H_{k^*}$ and $h^*$ is consistent with $S$.
If $k=k^*$, Claim~\ref{claim:kstar} below shows that $k$ will stay at $k^*$.
This proves the first part of the claim.

\item An iteration becomes unverified only if
$k$ gets incremented, and Algorithm~\ref{alg:aalarch} maintains the invariant that $k_i \leq k^*$.
Thus, the number of unverified iterations is at most $k^*$. In addition, each newly sampled set is
of size at most $N$. So the number of unverified examples is at most $k^*N$.

Hence, $i$---the total number of examples processed up to step $i$---equals the sum of the number of verified examples $l_i$, plus the number of unverified
examples, which is at most $k^*N$.
This proves the second part of the claim.
\qedhere
\end{enumerate}
\end{proof}

We show a technical claim used in the proof of Lemma~\ref{lem:k}
which guarantees that, on event $E$,
when $k$ has reached $k^*$, it will remain $k^*$ from then on. Recall
that $k_i$ is defined as the final value of $k$ at the end of step $i$;
 $i_k^0 = \min\cbr{i: k_i \geq k}$ is the step at the
end of which the working hypothesis space reaches level $\geq k$.

\begin{claim}
If $i_{k^*}^0$ is finite, then the following hold for all $i \geq i_{k^*}^0$:
\begin{enumerate}
\item[\textbf{(C1)}] $L_i$ has favorable bias.
\item[\textbf{(C2)}] Step $i$ terminates with $k_i = k^*$.
\item[\textbf{(C3)}] $h^* \in V_i$.
\end{enumerate}
Above, $L_i$ and $V_i$ denote their final values in \algaa.
\label{claim:kstar}
\end{claim}
\begin{proof}
By induction on $i$.
\paragraph{Base Case.} Let $i = i_{k^*}^0$.
Consider the execution of \algaa at the start
of step $i_{k^*}^0$ (line~\ref{step:aalarch-inc-i}).
Since by definition of $i_k$, the final value of $k$ at step
$i_{k^*}^0 - 1$ is $< k^*$, at step $i_{k^*}^0$, line~\ref{step:aalarch-err}
or line~\ref{step:aalarch-seed} is triggered.
Hence the dataset
$L_{i_{k^*}^0}$ equals some verified labeled dataset $L$ stored by \algaa,
i.e. $L_j$ for some $j \in \Gamma$. Thus, applying Lemma~\ref{lem:favbias}, Claim C1 holds.

We focus on the moment in step $i = i_{k^*}^0$ when $k$ increases to $k^*$ in
$\UVS$(line~\ref{step:aalarch-err-uvs} or~\ref{step:aalarch-seed-uvs}).
Now consider the temporary $V_{i_{k^*}^0}$ computed in the next line (\PVS).
Item 2 of Lemma~\ref{lem:hsgood} implies that
the version space $V_{i_{k^*}^0}$ is such that $\EC(V_{i_{k^*}^0}, L_{i_{k^*}^0}, \delta_{i_{k^*}^0})$ returns false.
Therefore the final value of $k$ in step $i_{k^*}^0$ is exactly $k^*$. Claim C2 follows.

Claim C2 implies the temporary $V_{i_{k^*}^0}$ is final. Item 1 of Lemma~\ref{lem:hsgood}
implies that $h^* \in V_{i_{k^*}^0}$, establishing Claim C3.

\paragraph{Inductive Case.} Now consider $i \geq i_{k^*}^0 + 1$. The inductive
hypothesis says that Claims C1--3 hold for step $i-1$.

Claim C1 follows from Claim C3 in step $i-1$. Indeed,
the newly added $x_i$ either comes
from the agreement region of $V_{i-1}$, in which case label $y_i$ agrees with
$h^*(x_i)$, or is from the disagreement region of $V_{i-1}$, in which case the
inferred label $y_i$
is queried from $\LABEL$. Following the same reasoning as the proof
of Lemma~\ref{lem:favbias}, Claim C1 is true.

Claims C2 and C3 follows the same reasoning as the proof for the base case.
\end{proof}

\begin{lemma}
If $i$ is in $\Gamma$, then $L_i$ has favorable bias. That is, for any hypothesis $h$,
\begin{equation*}
\err(h, L_i^D) - \err(h^*, L_i^D) \leq \err(h, L_i) - \err(h^*, L_i).
\end{equation*}
\label{lem:favbias}
\end{lemma}
\begin{proof}
We can split $L_i^D$ into two subsets, the subset where $L_i^D$ agrees with $L_i$ and
the subset $Q^D_i=\{ (x,y)\in L_i^D: h^*(x)\neq y\}$ where $L_i^D$ disagrees with $L_i$.
On the former subset, $L_i^D$ is identical to $L_i$, thus we just need to show that
\[
\err(h, Q_i^D) - \err(h^*, Q_i^D) \leq \err(h, Q_i) - \err(h^*,Q_i),
\]
where $Q_i = \{(x,y) : (x,-y) \in Q_i^D \}$.
Since $ \err(h^*, Q_i^D)= 1$ and $\err(h^*,Q_i)=0$, this reduces to showing
that $\err(h, Q_i^D) \leq 1 + \err(h, Q_i)$, which is easily seen to hold for any $h$ as $\err(h, Q_i^D)\leq 1$ and $\err(h, Q_i)\geq 0$.
\end{proof}

\end{document}